\documentclass{article}

\usepackage[numbers,sort,compress]{natbib}
\usepackage[preprint]{neurips_2026}

\usepackage{microtype}
\usepackage{graphicx}
\usepackage{subcaption}
\usepackage{booktabs}
\usepackage{hyperref}
\usepackage{enumitem}
\usepackage{tabularx}
\usepackage{multirow}
\usepackage{array}
\usepackage{wrapfig}

\usepackage{amsmath,amssymb,amsfonts,amsthm}
\usepackage{mathtools}
\usepackage{dsfont}
\usepackage{xcolor}

\usepackage{algorithm}
\usepackage{algorithmic}

\usepackage{caption}
\usepackage{subcaption}
\DeclareCaptionSubType{algorithm}

\usepackage[capitalize,noabbrev]{cleveref}
\usepackage{cleveref}
\crefname{subalgorithm}{Algorithm}{Algorithms}
\Crefname{subalgorithm}{Algorithm}{Algorithms}

\newtheorem{thm}{Theorem}[section]
\newtheorem{lem}[thm]{Lemma}

\newtheorem{observation}[thm]{Observation}

\newcommand{\closest}{\operatorname{closest}}
\newcommand{\cluster}{\operatorname{cluster}}
\newcommand{\cost}{\operatorname{cost}}
\newcommand{\NN}{\operatorname{NN}}
\DeclareMathOperator{\OPT}{OPT}

\newcommand{\tcp}[1]{\hfill$\triangleright$~#1}

\newcolumntype{C}[1]{>{\centering\arraybackslash}p{#1}}

\newcommand{\std}[1]{\ensuremath{\,{\scriptstyle\pm #1}}}

\newcommand{\pmstd}[2]{\ensuremath{#1\std{#2}}}

\newcommand{\scicost}[3]{(\ensuremath{#1\std{#2})\times 10^{#3}}}

\newcommand{\best}[1]{\ensuremath{\mathbf{#1}}}
\newcommand{\second}[1]{\ensuremath{\underline{#1}}}

\begin{document}


\title{
Simple KNN-Based Outlier Detection \\ Achieves Robust Clustering
}

\author{\textbf{Tianle Jiang},
  \textbf{Yufa Zhou}
  \\
  Duke University\\
  \texttt{\{tianle.jiang,yufa.zhou\}@duke.edu}
}



\maketitle

\begin{abstract}
Being robust to the presence of outliers is crucial for applying clustering algorithms in practice.
In the $\textit{robust $k$-Means}$ problem (i.e., $k$-Means with outliers), the goal is to remove $z$ outliers and minimize the $k$-Means cost on the remaining points.
Despite the close connection between robust $k$-Means and outlier detection, both theoretical and empirical understanding of the effectiveness of $\textit{classic outlier detection heuristics}$ for robust $k$-Means remains limited.
In this paper, we prove that under a practical assumption on the optimal cluster sizes, simply removing points with large $K$-Nearest-Neighbor distances achieves performance comparable to prior work in terms of approximation guarantees: it yields a constant-factor reduction from robust $k$-Means to standard $k$-Means, without introducing additional centers or discarding extra outliers, as is commonly required by existing approaches.
Empirically, experiments on real-world datasets show that our method outperforms or matches several more sophisticated algorithms in terms of clustering cost and runtime\footnote{Code: \url{https://github.com/MasterZhou1/Robust-Clustering}}. 
These results demonstrate that simple KNN-based heuristics can be surprisingly effective for robust clustering, highlighting new opportunities to bridge techniques from outlier detection and clustering.

\end{abstract}

\section{Introduction}

Clustering is a fundamental problem in unsupervised learning and has been studied extensively in machine learning, data mining, and theoretical computer science. However, it is well known that widely used clustering objectives such as $k$-Median and $k$-Means are highly sensitive to noise and outliers \citep{CharikarKMN01, ChawlaG13}. In particular, even a single data point that lies far from the underlying clusters can significantly distort the solution and lead standard clustering algorithms to select poor centers. Since noise and corrupted points are ubiquitous in real-world datasets, the design of efficient and practical clustering algorithms that are robust to outliers has attracted sustained attention \cite{ChawlaG13, GuptaKLMV17, ImQMSZ20, HuangFH0024}.

Most work in this area focuses on the robust $k$-Clustering problem (also known as $k$-Clustering with Outliers), first formalized by \citet{CharikarKMN01}. In this problem, the algorithm is given an outlier budget $z$ and aims to minimize a standard clustering objective, such as $k$-Median or $k$-Means, after removing at most $z$ points from the dataset. Although polynomial-time constant-factor approximation algorithms are known for both robust $k$-Median and robust $k$-Means \cite{Chen08, KrishnaswamyLS18, GuptaMZ21}, recent research has increasingly emphasized scalability. Specifically, there is a strong interest in finding robust $k$-Means algorithms that run in near-linear time in the number of data points $n$.

Existing efficient algorithms for robust $k$-Means can be broadly categorized into two classes:
(1) \emph{Outlier-removal-based methods} \cite{ChawlaG13, GuptaKLMV17, ImQMSZ20, GrunauR22} explicitly identify and discard suspected outliers, and then apply a standard clustering algorithm to the remaining points. These approaches typically have super-linear runtime and, in practice, rely on coreset constructions to scale to large datasets.
(2) \emph{Sampling-based methods} \cite{BhaskaraVX19, DeshpandeKP20, GrunauR22, HuangFH0024} avoid explicit outlier removal and instead use adaptive sampling strategies to select centers while reducing the probability of sampling outliers.

To achieve a favorable trade-off between efficiency and solution quality, algorithms in both categories generally resort to bicriteria approximations. That is, they allow the use of more than $k$ centers, the removal of more than $z$ outliers, or both. For instance, one of the most efficient recent algorithms by \citet{HuangFH0024} runs in $O(\epsilon^{-2} ndk \log\log n)$ time and achieves a $4$-approximation, but uses $O(k/\epsilon)$ centers and discards $(1+\epsilon)z$ outliers. We refer readers to \citet{HuangFH0024} for a comprehensive comparison of runtime and approximation guarantees.

Outlier detection (or anomaly detection) is a closely related but distinct problem, where the objective is to identify anomalous points rather than to optimize a clustering cost. A classic and effective approach in this literature ranks data points by their distance to the $K$-th nearest neighbor \cite{RamaswamyRS00}, which we name as the \emph{KNN heuristic}. Owing to its simplicity and strong empirical performance, this heuristic has become a widely used baseline across many application domains \cite{LiaoV02, ChandolaBK09, ZhangMH10}. Despite the conceptual connection between outlier detection and robust clustering, to the best of our knowledge, there has been no theoretical or systematic empirical study of applying the KNN heuristic as an outlier-removal step for the robust $k$-Means problem. While more sophisticated outlier detection techniques have been considered in \citet{HuangFH0024}, they typically result in substantially worse clustering costs compared to algorithms specifically designed for robust $k$-Means.

\subsection{Our Contributions}

In this paper, we investigate whether the KNN heuristic for outlier removal can achieve rigorous theoretical approximation guarantees and empirical practicality for the robust $k$-Means problem, which, surprisingly, is an overlooked direction in the robust clustering literature. 

\textbf{Theoretical Results.}
We note that several prior works \cite{ImQMSZ20, AlmanzaEPR22, HuangFH0024} adopt an assumption on the input instance that the size of each cluster in the optimal solution is lower bounded by $3z$. This is a well-justified practical assumption, since the number of outliers is typically small compared to the number of inliers, and inliers are often well clustered. The datasets that are commonly used to evaluate robust clustering algorithms also have the property that the true clusters are larger than the number of outliers. Actually, if some points form a very small cluster, it is hard to distinguish whether they are true inliers or a group of adversarial outliers.
Under this assumption, these works show that one can substantially reduce the violation in the number of centers used or the number of outliers removed (the comparison is displayed in~\Cref{tab:theoretical_comparison}). 

We prove that under the same assumption, the KNN heuristic, despite its extreme simplicity, already achieves strong theoretical guarantees. In particular, applying the vanilla KNN heuristic with $K = 2z$ to identify and remove $z$ outliers yields a $9$-approximate reduction from robust $k$-Means to standard $k$-Means. We emphasize that, while the parameter $K$ in the KNN heuristic is traditionally chosen to be a constant, our result requires $K$ to scale linearly with $z$ in order to obtain provable guarantees. Compared to \citet{ImQMSZ20}, whose method also gives a $9$-approximate reduction but requires access to the optimal cost $\OPT$ and removes $2z$ outliers, and to \citet{HuangFH0024}, who achieves a $4$-approximation but must use $O(k)$\footnote{Although not explicitly stated in their paper, the hidden constant in $O(\cdot)$ is very large.} centers, the KNN heuristic \textbf{strictly respects the constraints on both $k$ and $z$}. Although the assumption of knowing $\OPT$ can be removed via geometric guessing, doing so introduces significant overhead in practice.

\begin{table*}[t]
\caption{Comparison with previous theoretical results. Here $|P^*|$ denotes the size of the smallest cluster in the optimal solution.
}
\label{tab:theoretical_comparison}
\centering
\resizebox{\textwidth}{!}{
\begin{tabular}{cccccc}
\toprule
Algorithm & Approximation & Centers Used & Outliers Removed & Assumption & Source\\
\midrule
NKMeans & $9$ (reduction) & $k$ & $(3k+2)z$  & Know $\OPT$ & \citet{ImQMSZ20}\\
NKMeans & $9$ (reduction) & $k$ & $2z$  & $|P^*|\geq 3z$, Know $\OPT$ & \citet{ImQMSZ20}\\
TIKMeans & $4$ & $O(k/\epsilon)$ & $(1+\epsilon)z$  & No & \citet{HuangFH0024}\\
TIKMeans & $4$ & $O(k)$ & $z$  & $|P^*|\geq 3z$ & \citet{HuangFH0024}\\
OKMeans & $9$ (reduction) & $k$ & $z$  & $|P^*|\geq 3z$ & \Cref{thm:main_result}\\
OKMeans2 & $5.98$ (reduction) & $k$ & $z$  & $|P^*|\geq 3z$ & \Cref{thm:main_result2}\\
\bottomrule
\end{tabular}
}
\end{table*}

The assumed lower bound of $3z$ in prior work is somewhat arbitrary, and they do not provide a justification for the specific choice of $3$. Therefore, we extend our analysis and show in \cref{thm:main_result_informal} that, as long as the size of each cluster in the optimal solution is at least $cz$ for some constant $c > 1$, the vanilla KNN heuristic (with $K = \tfrac{c+1}{2}z$) yields a constant-factor reduction, with the approximation ratio improving as $c$ increases. 
\Cref{tab:theoretical_comparison} compares our results with prior work under the same assumption.

\begin{thm} [Informal]
\label{thm:main_result_informal}
     If each cluster in the optimal solution has size at least $cz$ for $c>1$, then using the vanilla KNN heuristic to remove $z$ outliers gives a $\Phi(c)$-approximate reduction to standard $k$-Means, where $\Phi(\cdot)$ is a monotone decreasing function. \footnote{The best approximation ratio we can obtain for each $c$ depends on the solution to a complex equation and has no succinct closed-form representation.}
\end{thm}

Inspired by the proof of~\Cref{thm:main_result_informal}, we propose a slight variant of the vanilla KNN heuristic that aggregates distances to multiple neighbors, rather than relying on a single neighbor, as the outlier score. This variant can be analyzed using a similar proof strategy and yields a better approximation guarantee in \cref{thm:main_result2_informal}, further demonstrating the versatility and potential of KNN-based outlier detection heuristics for the robust clustering problem. When each optimal cluster has size at least $3z$, the approximation ratio \textbf{improves from $9$ to $5.98$}.

\begin{thm} [Informal]
\label{thm:main_result2_informal}
     If each cluster in the optimal solution has size at least $cz$ for $c>1$, then using a variant of the KNN heuristic to remove $z$ outliers gives a $\Psi(c)$-approximate reduction to standard $k$-Means, where $\Psi(\cdot )$ is a monotone decreasing function.
\end{thm}

\textbf{Empirical Results.} 
We evaluate our proposed \cref{alg:OKMeans,alg:OKMeans2} on four real-world datasets ranging from small to large scale (up to 5 million points), comparing them against six baselines. Our experiments demonstrate that: (1) \textbf{Clustering Quality}: 
Our methods achieve competitive clustering costs and often match or outperform strong robust clustering baselines, especially on large-scale datasets such as SUSY and KDDFULL; notably, on SUSY-10, OKMeans improves upon the previous SOTA method, TIKMeans. 
(2) \textbf{Scalability}: By leveraging coresets and efficient nearest-neighbor search implementations, our approach scales effectively to millions of data points without sacrificing solution quality, validating that simple KNN heuristics can outperform complex sampling-based algorithms in practice. (3) \textbf{Parameter Sensitivity}: Empirically, the performance of our methods is not particularly sensitive to the choice of the parameter $c$. However, on some datasets, choosing $K$ as a small fixed constant in the KNN heuristic can lead to noticeably worse performance. 
Although choosing $K$ as a small constant is common practice when applying KNN to outlier detection, our result suggests that this choice may be suboptimal for the robust clustering objective.

\section{Related Work}

\subsection{(Robust) Clustering}
Two of the most classical clustering objectives are $k$-Median and $k$-Means. Both problems are NP-hard, but there exist constant-factor approximation algorithms for general metrics \cite{CharikarGTS99, KanungoMNPSW04, GowdaPST23, Cohen-Addad0LSS25, CharikarCGGLW25}, and improved approximation ratios \cite{AhmadianNSW17, Cohen-AddadEMN22} and even polynomial-time approximation schemes (PTAS) \cite{AroraRR98, Matousek00} in Euclidean space.
The class of robust clustering problems was first proposed by \citet{CharikarKMN01}. For robust $k$-Median, they gave a $4(1+1/\gamma)$-approximation algorithm when the algorithm is allowed to remove $(1+\gamma)z$ outliers. In subsequent work, the only known polynomial-time, non-bicriteria, constant-factor approximation algorithms \cite{Chen08, KrishnaswamyLS18, GuptaMZ21} are based on linear programming relaxations and iterative rounding, but their runtimes are high-order polynomials in $n$. The best approximation ratios known so far are a $(6.994+\epsilon)$-approximation for robust $k$-Median and a $(53.002+\epsilon)$-approximation for robust $k$-Means, both with running time $n^{O(\mathrm{poly}(1/\epsilon))}$.
Another recent line of research direction studies fixed-parameter tractable (FPT) approximation algorithms \cite{FriggstadKRS19, FengZHXW19, AgrawalISX23, ChenHXXZ23, Jaiswal023}, where $k$, $z$, or both are treated as parameters. However, the running times of these algorithms remain impractical even for moderate values of $k$ and $z$.

\subsection{Outlier Detection and KNN}

Nearest-neighbor--based methods form a classical paradigm for outlier detection. The idea of identifying outliers based on neighborhood distances first appeared in \citet{KnorrN98}, and the KNN heuristic based on a global ranking was later formalized by \citet{RamaswamyRS00}. Subsequently, numerous variants of the original KNN heuristic have been proposed \cite{EskinAPPS02, ZhangW06, OteyGP06, WeiQZJY03, KouLC06}, which either make more extensive use of $K$-nearest-neighbor information or extend KNN beyond purely numerical attributes.
There has also been substantial effort devoted to optimizing the runtime of computing or approximating KNN distances \cite{EskinAPPS02, GhotingPO08, TaoXZ06, WuJ06, MalkovY20, JohnsonDJ21, jayaram2019diskann, WangXY021, KrishnaswamyMS24}. More recently, KNN-based retrieval has been successfully incorporated into retrieval-augmented generation (RAG) systems \cite{GuoSLGSCK20, LewisPPPKGKLYR020, BorgeaudMHCRM0L22} to enhance the performance of large language models.

\section{Preliminaries}
\label{sec:prelim}

\subsection{Problem Definition}

\paragraph{Standard $k$-Clustering.}
Let $[n]:= \{1,2,3,\dots,n\}$ for $n \in \mathbb{N}^+$. 
In the standard $k$-clustering problem, the input consists of a set $X$ of $n$ points in a metric space and an integer parameter $k\in[n]$. The goal is to compute a set $C$ of $k$ centers, then each input point is assigned to its closest center. The quality of a clustering solution is measured by an objective function that depends on the distances between points and their assigned centers.
Let $d(\cdot,\cdot)$ denote the distance function of the metric space, and define $d(x,C)=\min_{c\in C} d(x,c)$. The three most common $k$-clustering objectives are:
(1) \textbf{$k$-Median:} $\cost(X,C)=\sum_{x\in X} d(x,C)$; 
(2) \textbf{$k$-Means:} $\cost(X,C)=\sum_{x\in X} d(x,C)^2$; 
(3) \textbf{$k$-Center:} $\cost(X,C)=\max_{x\in X} d(x,C)$.

In this paper, we focus on the $k$-Means objective. We will briefly discuss how our algorithms can be extended to $k$-Median and $k$-Center in Appendix~\ref{app:extension}. Note that there are different variants of the $k$-clustering problem regarding what points can be chosen as centers. The most common setting allows any point in the metric space to be chosen as a center, while there are also variants that require $C$ to be a subset of $X$, or a subset of a given set of candidate centers. The main theoretical result in this paper is indifferent among these settings.

\paragraph{Clustering with Outliers.}
In the robust $k$-Means problem, an instance $I=(X,k,z)$ consists of a dataset $X$ of $n$ points in a metric space and parameters $k\in[n]$ and $z\in[n]$. The goal is to select a subset $Z\subseteq X$ of $z$ outliers such that the $k$-Means cost of the remaining points is minimized, formally
$
\min_{Z\subseteq X,\ |Z|=z} \min_{C: |C|=k}\sum_{x\in X\setminus Z} d(x,C)^2 .
$
Equivalently, the problem can be viewed as selecting a set $C$ of $k$ centers, such that after removing the $z$ points in $X$ with the largest $d(x,C)$ as outliers, the $k$-Means cost of the remaining points is minimized.

\subsection{Notations}

\paragraph{Robust \texorpdfstring{$k$}{k}-Means Objective.}
We focus on the robust $k$-Means problem. For a dataset $X$, let $f(X) := \min_{C:|C|=k} f(X,C)$ denote the optimal $k$-Means cost over all choices of $k$ centers, where $f(X,C)$ is the $k$-Means cost induced by a fixed center set $C$ on $X$.  
Given an outlier budget $z$, we similarly define $f_z(X) := \min_{C:|C|=k} f_z(X,C)$ as the optimal robust $k$-Means cost, where $f_z(X,C) := \min_{Z\subseteq X,|Z|=z} \sum_{x\in X\setminus Z}d(x,C)^2$ denotes the robust $k$-Means cost achieved by a fixed set of centers $C$ on $X$. We use $g$ in place of $f$ when referring to analogous definitions for the $k$-Median objective.

\paragraph{Approximation Guarantee.}
An algorithm is said to achieve an $\alpha$-approximation if, for any instance $I=(X,k,z)$, it outputs a set $C$ of $k$ centers such that
$
f_z(X,C) \le \alpha \cdot f_z(X).
$
An algorithm is said to provide an $\alpha$-approximate reduction from robust $k$-Means to standard $k$-Means, if for any instance $I=(X,k,z)$, outputs a set of points $X'\subseteq X$ such that $|X'|=|X|-z$ and $f(X') \leq \alpha\cdot f_z(X)$.

\paragraph{Other Useful Notations.}

Let $X$ be a dataset and let $C$ be a set of $k$ centers. Given a prescribed number of outliers $z$, let $Z(C)$ denote the set of outliers induced by $C$, which consists of the $z$ points in $X$ with the largest distances to their closest centers in $C$. The corresponding set of inliers is $X_z(C) := X \setminus Z(C)$.
For any point $x$ and radius $r>0$, let $B_X(x,r) := \{y \in X \mid d(x,y) \leq r\}$ denote the set of points in $X$ that are within distance $r$ from $x$, and let $B_X^\circ(x,r) = \{y\in X: d(x,y)<r\}$. Let $\NN_X(x,K)$ denote the $K$-th nearest-neighbor of $x$ among $X$. If multiple neighbors are at the same distance to $x$, the ranking breaks ties arbitrarily.
For any inlier $x \in X_z(C)$, we define its closest center as $\closest(x,C) := \arg\min_{c \in C} d(x,c)$, breaking ties arbitrarily. The (robust) cluster of $x$ induced by $C$ is then defined as $\cluster_z(x,C) := \{y \in X_z(C) \mid \closest(y,C) = \closest(x,C)\}$.

\section{Main Results}
\label{sec:algo}


In this section, we present our main theoretical contribution: the classic and simple idea of identifying outliers based on nearest-neighbor distance gives a constant-factor reduction from robust $k$-Means to standard $k$-Means, under the practical assumption that the optimal clusters are large enough. We first analyze the vanilla version of the KNN heuristic, then propose and analyze a slight variant inspired by the first analysis, which gives a better approximation ratio.

\subsection{Analysis of Vanilla KNN}

\begin{thm} [Restatement of~\Cref{thm:main_result_informal}]
\label{thm:main_result}
    Let $I=(X,k,z)$ be an instance of robust $k$-Means and $C^*$ be the optimal set of centers. Suppose $\mathcal{A}$ is a $\beta$-approximation algorithm for $k$-Means, and each cluster induced by $C^*$ has size at least $cz$, then there is an algorithm that calls $\mathcal{A}$ once and computes a set of centers $C$ such that $f_{z}(X,C) \leq \Phi(c)\beta\cdot f_z(X,C^*)$, 
    where $\Phi(c)=\frac{(1+\sqrt{(c-1)/2})^2}{(c-1)/2}(x^*)^2$ and $x^*$ is the unique real root greater than $1$ of the following equation:
    \[
    \left[\left(1 + \sqrt{\frac{c-1}{2}}\right)^2x^2 - \frac{c-1}{2}\right](x-1)^2 - 2x+1=0 ~.
    \]
\end{thm}

Here we only present a proof for the special case of $c=3$ (where $\Phi(3)=9$), which aligns with the assumptions from prior work \cite{ImQMSZ20,HuangFH0024} and is numerically simpler. The proof for general $c$ follows the same structure but requires careful tuning of the intermediate parameters used in the analysis and is deferred to~\Cref{app:proof_main_result}.
The formal algorithm is presented in~\Cref{alg:OKMeans}. In the special case of $c=3$, for each point $x$, we simply compute the distance to its $2z$-th nearest-neighbor, and remove the $z$ points with largest $2z$-NN distances.


\begin{algorithm}[t]
\caption{Two KNN-based outlier detection algorithms for robust $k$-means. 
$X$ is the input point set, $k$ is the number of clusters, $z$ is the outlier budget,
$A(\cdot,k)$ is a standard $k$-means algorithm, and $c>1$ is the constant in the cluster size assumption.}
\label{alg:okmeans_both}

\subcaptionbox{\textbf{OKMeans}$(X,k,z)$\label{alg:OKMeans}}[0.49\linewidth]{
\begin{minipage}[t]{\linewidth}
\begin{algorithmic}[1]
\STATE \textbf{for each} $x\in X$ \textbf{do}
\STATE \quad $r_x \gets d(x, \NN_X(x,\frac{c+1}{2}z))$ 
\footnotemark
\STATE \tcp{radius to the $2z$-th neighbor}
\STATE \textbf{end for}
\STATE Let $O\subseteq X$ be the $z$ points with largest $r_x$
\STATE \textbf{return} $A(X\setminus O,k)$
\end{algorithmic}
\end{minipage}
}
\hfill
\subcaptionbox{\textbf{OKMeans2}$(X,k,z)$\label{alg:OKMeans2}}[0.49\linewidth]{
\begin{minipage}[t]{\linewidth}
\begin{algorithmic}[1]
\STATE \textbf{for each} $x\in X$ \textbf{do}
\STATE \quad $s_x \gets \sum_{i=z+1}^{cz} d\!\left(x,\NN_X(x,i)\right)$ 
\STATE \tcp{aggregate distance to mid-range neighbors}
\STATE \textbf{end for}
\STATE Let $O\subseteq X$ be the $z$ points with largest $s_x$
\STATE \textbf{return} $A(X\setminus O,k)$
\end{algorithmic}
\end{minipage}
}

\end{algorithm}
\footnotetext{We assume without loss of generality that $\frac{c+1}{2}z$ is an integer. Our implementation uses $\left\lfloor \frac{c+1}{2}z \right\rfloor$, which has a negligible impact on the analysis.}

We fix an optimal set of centers $C^*$, and call the points in $X_z(C^*)$ ``true inliers'' and the points in $X\setminus X_z(C^*)$ ``true outliers''. Define $Y=O\cap X_z(C^*)$ as the set of true inliers removed by the algorithm by mistake, and $Z'=(X\setminus O) \cap (X\setminus X_z(C^*))$ as the set of true outliers included by the algorithm by mistake, then we have $|Y|=|Z'|$. Below we assume that $|Y|\geq 1$, otherwise we have already identified all outliers correctly and the reduction is costless. Let $T = \min_{x\in O} \{r_x\}$, i.e., the threshold separating the $r_x$ of $X\setminus O$ and $O$. We start by proving two lower bounds on the cost of the optimal solution.


\begin{lem}
\label{lem:opt_lowerbound1}
    Suppose each optimal cluster has size at least $3z$, and $\min_{y\in Y}d(y,C^*) = \lambda T$ for some $\lambda\geq 0$, then $f_z(X,C^*) \geq |Y|T^2\cdot \max\{\lambda^2, (1-\lambda)^2\}$.
\end{lem}
\begin{proof}
    Since every $y\in Y$ satisfies $d(y,C^*) \geq \lambda T$, we have
    \[
    f_z(X,C^*) \geq f_z(Y,C^*) = \sum_{y\in Y}d(y,C^*)^2 \geq |Y|T^2\cdot \lambda^2 ~.
    \]
    If $\lambda \geq 1$ then $\lambda^2 \geq (1-\lambda)^2$ and we are done. Otherwise, fix a $y\in Y$ such that $d(y,C^*) =\lambda T$ for some $\lambda < 1$. Since $y\in O$, we have $r_y \geq T$, which implies $|B_X^\circ(y,T)| \leq 2z$. By the assumption that each optimal cluster has size at least $3z$, at least $z$ points in $\cluster_z(y,C^*)$ are not in $B_X^\circ(y,T)$, and for each such point $x$, by the triangle inequality, we have
    \[
    \begin{aligned}
    d(x,\closest(y,C^*)) \geq &~ d(x,y) - d(y,\closest(y,C^*))\\
    \geq &~ T-\lambda T = T(1-\lambda) ~,
    \end{aligned}
    \]
    therefore,
    \[
    \begin{aligned}
    f_z(X,C^*) \geq &~ f_z(\cluster_z(y,C^*) \setminus B_X^\circ(y,T), C^*) \\
    \geq &~ zT^2\cdot (1-\lambda)^2 \geq |Y|T^2\cdot (1-\lambda)^2 ~.
    \end{aligned}
    \]
    The last inequality follows from the fact that $|Y| \leq |O|= z$. \qedhere
\end{proof}

\begin{lem}
\label{lem:opt_lowerbound2}
    Suppose each optimal cluster has size at least $3z$, and $\max_{o\in Z'}d(o,C^*) = \delta T$ for some $\delta>1$, then $f_z(X,C^*) \geq |Y|T^2\cdot (\delta-1)^2$.
\end{lem}
\begin{proof}
    Let $o = \arg\max_{o'\in Z'}\{d(o',C^*)\}$, so $d(o,C^*)=\delta T$. Since $o\in X\setminus O$, we have $r_o \leq T$, which implies $|B_X(o,T)|\geq 2z$. Therefore, there must be at least $z$ true inliers in $B_X(o,T)$. For each such inlier $x$, we must have $d(x,C^*) \geq (\delta-1)T$, otherwise there would be the following contradiction: 
    \[
    d(o,C^*) \leq d(o,x)+d(x,C^*) <   T + (\delta-1)T = \delta T ~.
    \]
    Thus, we must have
    \[
    f_z(X,C^*) \geq \sum_{x\in B_X(o,T)\cap X_z(C^*)} d(x,C^*)^2
    \geq zT^2 \cdot (\delta-1)^2 \geq |Y|T^2\cdot (\delta-1)^2 ~. \qedhere
    \]
\end{proof}

\begin{lem}
\label{lem:cost_upperbound1}
    Suppose each optimal cluster has size at least $3z$, then $f(X\setminus O, C^*) \leq 9\cdot f_z(X,C^*)$~.
\end{lem}
\begin{proof}
    We distinguish between the following three cases:

    \paragraph{Case 1} If $\max_{o\in Z'}d(o,C^*) \leq \frac{3}{2}T$, then by~\Cref{lem:opt_lowerbound1}, assuming $\min_{y\in Y}d(y,C^*) = \lambda T$ for some $\lambda\geq 0$, then we have
    \[
    \begin{aligned}
        f(X\setminus O,C^*)
        \leq &~ f_z(X,C^*) + f(Z',C^*) - f(Y,C^*)\\
        \leq &~ f_z(X,C^*) + |Z'|\cdot \left(3T/2\right)^2 - |Y|\cdot (\lambda T)^2\\
        \leq &~ f_z(X,C^*)\cdot \left(1 + \frac{\left(\frac{3}{2}\right)^2-\lambda^2}{\max\{\lambda^2, (1-\lambda)^2\}}\right) ~,
    \end{aligned}
    \]
    which is maximized at $\lambda =\frac{1}{2}$ and takes value $9\cdot f_z(X,C^*)$.
    
    \paragraph{Case 2} If $\max_{o\in Z'}d(o,C^*) = \delta T$ for some $\delta > \frac{3}{2}$, and $d(y,C^*) \geq (\delta-1)T$ for all $y\in Y$, then by~\Cref{lem:opt_lowerbound1}, we have
    \[
    \begin{aligned}
        f(X\setminus O,C^*)
        \leq &~ f_z(X,C^*) + f(Z',C^*) - f(Y,C^*)\\
        \leq &~ f_z(X,C^*) + |Z'|\cdot (\delta T)^2 - |Y|\cdot ((\delta-1)T)^2\\
        \leq &~ f_z(X,C^*)\cdot \left(1 + \frac{2\delta-1}{(\delta-1)^2}\right) ~,
    \end{aligned}
    \]
    which is at most $9\cdot f_z(X,C^*)$ when $\delta > \frac{3}{2}$.
    
    \paragraph{Case 3} If $\max_{o\in Z'}d(o,C^*) = \delta T$ for some $\delta > \frac{3}{2}$, and $\min_{y\in Y}d(y,C^*) =\lambda T$ for some $0\leq \lambda < \delta-1$, then by~\Cref{lem:opt_lowerbound1} and~\Cref{lem:opt_lowerbound2}, we have
    \[
    \begin{aligned}
        f(X\setminus O,C^*)
        \leq &~ f_z(X,C^*) + f(Z',C^*) - f(Y,C^*)\\
        \leq &~ f_z(X,C^*) + |Z'|\cdot (\delta T)^2 - |Y|\cdot (\lambda T)^2\\
        \leq &~ f_z(X,C^*)\cdot \left(1 + \frac{\delta^2-\lambda^2}{\max\{(1-\lambda)^2, (\delta-1)^2\}}\right) ~,
    \end{aligned}
    \]
    which is at most $9 \cdot f_z(X,C^*)$ when $\delta > \frac{3}{2}$ and $\lambda < \delta-1$.
\end{proof}

Finally, we apply the $\beta$-approximate $k$-Means algorithm on $X\setminus O$, and obtain a set of centers $C$ such that $f(X\setminus O,C) \leq 9\beta\cdot f_z(X,C^*)$. If we then recompute the set of inliers with respect to $C$, the clustering cost does not increase, therefore, 
\[
f_z(X,C) \leq f(X\setminus O,C) \leq 9\beta\cdot f_z(X,C^*) ~.
\]
This concludes the proof of~\Cref{thm:main_result} for $c=3$.

\subsection{A New Variant of KNN}

Numerous variants that utilize more information about the K-neighborhood have emerged after the proposal of the original KNN heuristic. For example, instead of using the distance to the $K$-th nearest neighbor, \citet{EskinAPPS02} used the sum of distances to the $K$ nearest neighbors. 

We propose another variant that can be analyzed using the same framework as~\Cref{thm:main_result} and admits better approximation ratios for all $c$.
This variant is inspired by an attempt to analyze the aforementioned variant that identifies outliers based on the sum of distances to the $K$ nearest neighbors. We observe that it is not robust to adversarially generated input: the true outliers can be very close to each other, gaining an advantage when comparing the sum of distances, which leads to an even worse approximation ratio. As a result, we propose to omit the distances to the nearest $z$ neighbors, so that outliers cannot ``help'' each other achieve comparatively smaller total distances. 
We present the formal algorithm in~\Cref{alg:OKMeans2}: for each point $x$, we compute the total distance $s_x$ to its $K$-nearest-neighbors for all $K\in[z+1,cz]$, and remove the $z$ points with largest $s_x$ as outliers. 



\begin{thm} [Restatement of~\Cref{thm:main_result2_informal}]
\label{thm:main_result2}
    Let $I=(X,k,z)$ be an instance of robust $k$-Means and $C^*$ be the optimal set of centers. Suppose $\mathcal{A}$ is a $\beta$-approximation algorithm for $k$-Means, and each cluster induced by $C^*$ has size at least $cz$, then there is an algorithm that computes a set of centers $C$ such that $f_{z}(X,C) \leq \Psi(c)\beta\cdot f_z(X,C^*)$,
    where $\Psi(c)=\frac{(1+\sqrt{c-1})^2}{c-1}(x^*)^2$ and $x^*$ is the unique real root greater than $1$ of the following equation:
    \[
    \left[\left(1 + \sqrt{c-1}\right)^2x^2 - (c-1)\right](x-1)^2 - 2x+1=0 ~.
    \]
\end{thm}

The proof is deferred to~\Cref{app:proof_main_result2}. In particular, for $c=3$ this gives a $\Psi(3)\approx 5.98$-approximate reduction from robust $k$-Means to standard $k$-Means, which is significantly better than the $9$-approximate reduction given by~\Cref{thm:main_result}.
In \cref{tab:c_ratio} we list and compare the approximation ratios obtained by~\Cref{thm:main_result} and~\Cref{thm:main_result2} for some values of $c$. The approximation ratios improve steadily as the ``large-cluster" assumption becomes stronger.

\subsection{Coreset and Runtime Analysis}

The runtime of our proposed methods is $O(dn^2)$ when they are directly applied to the full dataset, dominated by the computation of KNN distances, which is impractical for large datasets.
Therefore, computing a coreset to reduce the data size is useful. We call a subset of points an $\alpha$-coreset if any $\beta$-approximate solution on the coreset yields an $\alpha\beta$-approximate solution on the original dataset.

\begin{thm}[\cite{ImQMSZ20}]
    There is an algorithm that computes an $O(1)$-coreset of size $O(k\log n)$ in $O(dkn\log n)$ time.
\end{thm}

Thus, the full algorithm first runs the above coreset construction and then applies our proposed methods, Algorithm~\ref{alg:OKMeans} or~\ref{alg:OKMeans2}, on the coreset to obtain the $k$ centers. The KNN computation on the coreset takes $O(d(k\log n)^2)$ time, which is dominated by the coreset construction cost. Therefore, the overall runtime is $O(dkn\log n)$, which is near-linear in the input size.

\section{Experiments}
\label{sec:experiments}

\subsection{Experimental Setup}

\textbf{Datasets.}
We evaluate our method on four real-world datasets: SKIN (245{,}057 $\times$ 3, $k = 10$), SUSY (5{,}000{,}000 $\times$ 18, $k = 10$), SHUTTLE (43{,}500 $\times$ 9, $k = 10$), and KDDFULL (4{,}898{,}431 $\times$ 37, $k = 3$), following \citet{HuangFH0024,ImQMSZ20,DeshpandeKP20}.
For outlier generation, we adopt the experiment settings from \citet{HuangFH0024}. 
We first normalize each dataset so that every dimension has mean 0 and standard deviation 1.
For SKIN and SUSY, we inject 1\% synthetic outliers uniformly sampled from the hypercube $[-\xi, \xi]^d$ with $\xi \in \{5, 10\}$, yielding the datasets SKIN-5, SKIN-10, SUSY-5, and SUSY-10.
For SHUTTLE, which contains 43{,}500 points with 99.6\% belonging to the four largest classes, we treat the two smallest classes (17 points total) as outliers.
For KDDFULL, which has 4{,}898{,}431 points with 99.07\% concentrated in the three largest classes, we designate the remaining 23 smaller classes (45{,}747 points) as outliers.

\textbf{Baseline Algorithms.}
We compare our algorithm with the following baselines:
\begin{itemize}[leftmargin=*]
    \item TIKMeans \cite{HuangFH0024}, which runs $O(k/\varepsilon)$ sampling iterations and then applies a center-reduction procedure with the greedy weighted clustering method of \citet{BhaskaraVX19}.
    \item IKMeans \cite{HuangFH0024}, which performs $O(k/\varepsilon)$ sampling iterations and then applies weighted $k$-means++.
    \item RobustKMeans++ \cite{DeshpandeKP20}, which samples $O(kn/z)$ candidate centers and applies weighted $k$-means++ to obtain the final clustering.
    \item NKMeans \cite{ImQMSZ20}, a density-based sampling method that filters points far from high-density regions and runs clustering on the resulting coreset. 
    \item KMeans++ \cite{arthur2007k}, which is the most standard outlier-unaware heuristic for $k$-Means.
\end{itemize}


\textbf{Implementation Details.}
All experiments are conducted on a server equipped with an AMD EPYC 7532 32-Core Processor and 768 GB of RAM. We reimplement all baseline algorithms from scratch and open-source the code. To ensure a fair comparison, we optimize all implementations using standard acceleration techniques, including vectorization, multi-threading, and BLAS-optimized matrix multiplication \cite{harris2020array} (computing distances via $\|x-y\|^2 = \|x\|^2 + \|y\|^2 - 2x^\top y$ with pre-computed norms).
To accelerate computationally intensive methods (specifically, OKMeans (both variants), NKMeans, and KMeans++), we utilize coreset construction via simple uniform sampling. This reduces the input to a smaller representative set, with the target outlier count scaled proportionally to the sampling ratio. The coreset sizes are set to 10,000 for SHUTTLE, 24,000 for the SKIN datasets, and 70,000 for the SUSY datasets and KDDFULL. 

\textbf{Experimental Details.}
We conduct two sets of experiments to demonstrate the effectiveness of our proposed methods, OKMeans (\Cref{alg:OKMeans}) and OKMeans2 (\Cref{alg:OKMeans2}). (1) We benchmark all algorithms on the six datasets described above. 
For the baseline algorithms, we follow exactly the experimental settings from prior works~\cite{HuangFH0024,ImQMSZ20,DeshpandeKP20,ChawlaG13}. 
Our \cref{alg:OKMeans,alg:OKMeans2} rely on $K$-nearest-neighbor search to compute pointwise distances, and we employ Scikit-Learn \cite{scikit-learn} for modest datasets (SKIN-5, SKIN-10, SHUTTLE) and FAISS \cite{douze2025faiss} for larger datasets (SUSY-5, SUSY-10, KDDFULL) to obtain better runtime. When comparing with baseline algorithms, we set the parameter $c$ to be $3$ (recall that $c$ affects the choice of $K$ in KNN).
(2) We perform a parameter sensitivity analysis by varying the parameter $c$. We use the FAISS implementation for our two algorithms in this experiment, setting $c \in \{3,4,5,7,10,15,20\}$. We further experiment with the classic setting of the KNN heuristic, where $K$ is usually chosen to be a constant, using $K\in\{2,5,10\}$.

\subsection{Results and Analysis}\label{sec:experiment:results}

\begin{table*}[!t]
\centering
\caption{Comparison results of robust $k$-means algorithms on each dataset. Best values are in bold, and second-best values are underlined. See more details in \cref{sec:experiment:results}.
}
\label{tab:baseline_mean_best}
\footnotesize
\setlength{\tabcolsep}{3.2pt}
\renewcommand{\arraystretch}{0.92}
\resizebox{\textwidth}{!}{%
\begin{tabular}{l c c c c @{\hspace{5mm}} l c c c c}
\toprule
Method & Cost Best $\downarrow$ & Cost Mean $\pm$ Std $\downarrow$ & Recall $\uparrow$ & Time(s) $\downarrow$ & Method & Cost Best $\downarrow$ & Cost Mean $\pm$ Std $\downarrow$ & Recall $\uparrow$ & Time(s) $\downarrow$ \\
\midrule

\multicolumn{5}{c}{\textbf{SKIN-5}} & \multicolumn{5}{c}{\textbf{SKIN-10}} \\
\cmidrule(lr){1-5}\cmidrule(lr){6-10}
TIKMeans & $5.854 \times 10^{4}$ & \best{\scicost{6.151}{0.194}{4}} & $0.742$ & $3.314$ & TIKMeans & \best{6.303 \times 10^{4}} & \best{\scicost{6.608}{0.131}{4}} & $0.940$ & $2.965$ \\
IKMeans & $6.218 \times 10^{4}$ & $\scicost{7.109}{0.885}{4}$ & \best{0.767} & \second{1.444} & IKMeans & $6.507 \times 10^{4}$ & $\scicost{7.512}{1.133}{4}$ & \second{0.950} & \second{1.368} \\
RobustKmeans++ & $6.679 \times 10^{4}$ & $\scicost{7.620}{0.711}{4}$ & $0.741$ & $1.618$ & RobustKmeans++ & $7.391 \times 10^{4}$ & $\scicost{8.857}{1.237}{4}$ & $0.936$ & $1.617$ \\
NKMeans & $5.830 \times 10^{4}$ & $\scicost{6.887}{0.646}{4}$ & $0.744$ & $4.402$ & NKMeans & $6.628 \times 10^{4}$ & $\scicost{7.057}{0.275}{4}$ & $0.944$ & $4.375$ \\
KMeans++ & \second{5.829 \times 10^{4}} & $\scicost{7.522}{1.464}{4}$ & $0.749$ & \best{0.476} & KMeans++ & $7.910 \times 10^{4}$ & $\scicost{1.018}{0.149}{5}$ & $0.891$ & \best{0.480} \\
OKMeans (Ours) & $6.060 \times 10^{4}$ & $\scicost{6.494}{0.312}{4}$ & \second{0.757} & $2.056$ & OKMeans (Ours) & \second{6.472 \times 10^{4}} & $\scicost{6.982}{0.671}{4}$ & $0.949$ & $2.045$ \\
OKMeans2 (Ours) & \best{5.768 \times 10^{4}} & \second{\scicost{6.410}{0.300}{4}} & $0.746$ & $2.395$ & OKMeans2 (Ours) & $6.490 \times 10^{4}$ & \second{\scicost{6.707}{0.166}{4}} & \best{0.955} & $2.441$ \\
\midrule

\multicolumn{5}{c}{\textbf{SUSY-5}} & \multicolumn{5}{c}{\textbf{SUSY-10}} \\
\cmidrule(lr){1-5}\cmidrule(lr){6-10}
TIKMeans & $5.103 \times 10^{7}$ & $\scicost{5.165}{0.048}{7}$ & $0.759$ & $94.824$ & TIKMeans & $5.197 \times 10^{7}$ & $\scicost{5.272}{0.040}{7}$ & $0.975$ & $89.186$ \\
IKMeans & $5.145 \times 10^{7}$ & $\scicost{5.279}{0.104}{7}$ & $0.782$ & $46.489$ & IKMeans & $5.194 \times 10^{7}$ & $\scicost{5.334}{0.087}{7}$ & $0.976$ & $42.134$ \\
RobustKmeans++ & $5.640 \times 10^{7}$ & $\scicost{5.791}{0.132}{7}$ & $0.708$ & $74.735$ & RobustKmeans++ & $5.759 \times 10^{7}$ & $\scicost{5.966}{0.130}{7}$ & $0.973$ & $74.092$ \\
NKMeans & $4.795 \times 10^{7}$ & \second{\scicost{4.825}{0.022}{7}} & $0.739$ & $17.346$ & NKMeans & $4.905 \times 10^{7}$ & $\scicost{4.941}{0.028}{7}$ & $0.972$ & $17.529$ \\
KMeans++ & $4.789 \times 10^{7}$ & $\scicost{4.885}{0.087}{7}$ & \best{0.831} & \best{9.725} & KMeans++ & $4.900 \times 10^{7}$ & $\scicost{5.173}{0.235}{7}$ & $0.980$ & \best{9.728} \\
OKMeans (Ours) & \best{4.756 \times 10^{7}} & $\scicost{4.830}{0.064}{7}$ & \second{0.799} & \second{12.840} & OKMeans (Ours) & \best{4.836 \times 10^{7}} & \best{\scicost{4.873}{0.023}{7}} & \second{0.982} & \second{12.919} \\
OKMeans2 (Ours) & \second{4.756 \times 10^{7}} & \best{\scicost{4.801}{0.034}{7}} & $0.798$ & $13.469$ & OKMeans2 (Ours) & \second{4.848 \times 10^{7}} & \second{\scicost{4.876}{0.020}{7}} & \best{0.982} & $13.407$ \\
\midrule

\multicolumn{5}{c}{\textbf{SHUTTLE}} & \multicolumn{5}{c}{\textbf{KDDFULL}} \\
\cmidrule(lr){1-5}\cmidrule(lr){6-10}
TIKMeans & \best{5.710 \times 10^{4}} & \best{\scicost{6.122}{0.260}{4}} & $0.141$ & $0.543$ & TIKMeans & $3.318 \times 10^{7}$ & \best{\scicost{3.320}{0.002}{7}} & $0.601$ & $49.850$ \\
IKMeans & $6.220 \times 10^{4}$ & $\scicost{7.277}{0.503}{4}$ & $0.141$ & \best{0.243} & IKMeans & $3.324 \times 10^{7}$ & $\scicost{3.937}{0.861}{7}$ & $0.615$ & $21.419$ \\
RobustKmeans++ & $6.246 \times 10^{4}$ & $\scicost{7.830}{1.193}{4}$ & $0.094$ & $0.363$ & RobustKmeans++ & $3.448 \times 10^{7}$ & $\scicost{4.845}{1.260}{7}$ & $0.611$ & $40.355$ \\
NKMeans & $6.487 \times 10^{4}$ & $\scicost{7.260}{0.353}{4}$ & \best{0.171} & $4.785$ & NKMeans & $3.323 \times 10^{7}$ & \second{\scicost{3.466}{0.178}{7}} & \best{0.624} & $14.859$ \\
KMeans++ & \second{6.048 \times 10^{4}} & $\scicost{7.055}{0.705}{4}$ & $0.082$ & \second{0.294} & KMeans++ & $3.321 \times 10^{7}$ & $\scicost{6.123}{2.152}{7}$ & $0.553$ & \best{6.894} \\
OKMeans (Ours) & $6.262 \times 10^{4}$ & \second{\scicost{6.823}{0.625}{4}} & \second{0.171} & $0.785$ & OKMeans (Ours) & \best{3.308 \times 10^{7}} & $\scicost{4.019}{0.770}{7}$ & \second{0.616} & \second{9.510} \\
OKMeans2 (Ours) & $6.262 \times 10^{4}$ & $\scicost{6.823}{0.625}{4}$ & $0.171$ & $0.776$ & OKMeans2 (Ours) & \second{3.308 \times 10^{7}} & $\scicost{4.250}{0.919}{7}$ & $0.613$ & $10.269$ \\
\bottomrule
\end{tabular}%
}
\end{table*}


\Cref{tab:baseline_mean_best} reports the clustering cost, outlier recall, and running time of all methods across datasets. For each algorithm, we define the \textit{Cost Best} as the minimum clustering cost achieved over 10 independent runs. The means of \textit{Recall} ($|\text{detected} \cap \text{true outliers}| / |\text{true outliers}|$) and \textit{Time} are reported in the table. We also report \textit{Cost Mean $\pm$ Std} to demonstrate the stability of the solutions found. We provide the full table of cost, recall, and runtime (means and stds) in \cref{tab:baseline_mean_std}. 

\textbf{Performance Comparison.}
As shown in \Cref{tab:baseline_mean_best}, our \textit{OKMeans} and \textit{OKMeans2} demonstrate a good balance of clustering quality and computational efficiency. 
\begin{itemize}[leftmargin=*]
    \item \textbf{Clustering Quality:} 
    Our methods achieve competitive clustering costs across all datasets and often match or outperform strong robust clustering baselines, especially on the large-scale SUSY and KDDFULL settings.
    For instance, on the large-scale SUSY-10 dataset, \textit{OKMeans} achieves a mean cost of $4.873 \times 10^7$, improving upon \textit{TIKMeans} ($5.272 \times 10^7$) and \textit{NKMeans} ($4.941 \times 10^7$). 
    \item \textbf{Runtime Efficiency:} While the classic \textit{KMeans++} is generally the fastest, it is highly unstable and fails to optimize the robust objective effectively. This instability motivates the need for robust algorithms. 
    On large-scale datasets, our methods are substantially faster than several robust baselines. For instance, on SUSY-10, OKMeans requires only $\approx 12.9$ seconds, compared with RobustKmeans++ ($\approx 74.1$s) and TIKMeans ($\approx 89.2$s). On KDDFULL, OKMeans is nearly $5\times$ faster than TIKMeans while achieving a lower best cost.
    \item \textbf{Scalability:} The results on large datasets (SUSY, KDDFULL) highlight the effectiveness of our FAISS-accelerated implementation. While \textit{TIKMeans} performs well on small data (SKIN), its runtime degrades severely on larger datasets. In contrast, \textit{OKMeans} maintains high scalability without sacrificing solution quality while maintaining competitive clustering quality.
\end{itemize}

\begin{wraptable}{r}{0.52\textwidth}
\vspace{-8pt}
\centering
\caption{Performance of the KNN heuristic on KDDFULL with constant values of $K$.}
\label{tab:kdd_metrics_bad}
\scriptsize
\setlength{\tabcolsep}{3pt}
\renewcommand{\arraystretch}{1.15}
\resizebox{\linewidth}{!}{
\begin{tabular}{l c c c c}
\toprule
Metric & OKMeans & KNN ($K=2$) & KNN ($K=5$) & KNN ($K=10$) \\
\midrule
Cost & $4.019 \times 10^{7}$ & $4.612 \times 10^{7}$ & $4.371 \times 10^{7}$ & $4.631 \times 10^{7}$ \\
Recall & 0.616 & 0.5605 & 0.5208 & 0.5693 \\
Time (s) & 9.510 & 9.141 & 9.273 & 9.104 \\
\bottomrule
\end{tabular}
}
\vspace{-8pt}
\end{wraptable}

\textbf{Parameter Sensitivity.}
For the KNN heuristic with \textit{constant values} of $K$, we observe a significant increase in cost on the KDDFULL dataset, as demonstrated in \cref{tab:kdd_metrics_bad}. The performance on other datasets is comparable with OKMeans and is deferred to \cref{tab:param_okmeans_bad}.
This suggests that for the robust $k$-Means objective, selecting $K$ to be a small fixed constant might be suboptimal in certain cases. 
The empirical results align with our theoretical finding that $K$ should scale linearly in the number of outliers to guarantee consistent good performance.

In \cref{app:figures}, \Cref{tab:param_scaling} presents the sensitivity of our methods to the parameter $c$. We provide additional visualizations in \cref{fig:param_cost_comparison,fig:param_recall_comparison,fig:param_time_comparison}. We observe that \textit{performance is robust to the choice of $c$}. Moderate values (e.g., $c \in [3, 7]$) are generally sufficient to reach the ``sweet spot'' where the cost is minimized. For example, on SUSY-5, OKMeans achieves its best cost at $c=5$ ($4.814 \times 10^7$), with minimal variance across other settings. Increasing $c$ further to 15 or 20 yields diminishing returns and does not significantly improve the clustering objective.

\section{Conclusion}

In this paper, we establish a direct connection between the classical KNN-based outlier detection heuristic and robust clustering. We provide the first rigorous theoretical analysis for KNN, explaining its strong empirical performance, and show experimentally that it matches more sophisticated methods. Our results highlight that a surprisingly simple heuristic can be highly effective, offering a new perspective on robust clustering algorithm design.

\begin{ack}
Tianle Jiang is supported by NSF grant IIS-2402823.
\end{ack}

\bibliography{main}

\bibliographystyle{plainnat}

\newpage
\appendix

\begin{center}
     {\bf \Large Appendix}
\end{center}
\section{Future Work}
\label{app:future_work}

Our findings suggest that outlier detection and robust clustering should be studied in a unified framework, leading to several immediate directions for future work:
\begin{enumerate}[leftmargin=*, topsep=2pt,itemsep=1pt,parsep=0pt]
    \item \textbf{Tighten the analysis of KNN.} 
    Our current analysis of the KNN heuristic is not tight. Establishing sharper bounds or identifying matching lower bounds remains an open problem.
    
    \item \textbf{Outlier detection for robust clustering.} 
    It would be interesting to analyze other widely used outlier detection heuristics, particularly variants of the KNN heuristic, and investigate whether they admit strong theoretical guarantees or empirical performance for the Robust $k$-Means objective.
    
    \item \textbf{Robust clustering for outlier detection.} 
    Another promising direction is to adapt advanced techniques from the robust clustering literature to design more effective outlier detection algorithms.
    
    \item \textbf{Beyond worst-case analysis.} 
    Our results focus on worst-case approximation guarantees, where both cluster structure and outliers are adversarial. Exploring refined analyses under additional, practically motivated assumptions on the data generation process may yield significantly stronger guarantees.
\end{enumerate}

\section{Deferred Proofs}\label{app:proof_main}

\subsection{Full Proof of~\Cref{thm:main_result}}
\label{app:proof_main_result}

We extend the analysis to general $c>1$. To justify the choice of $r_x = d(x, \NN_X(x,\frac{c+1}{2}z))$ in~\Cref{alg:OKMeans}, we temporarily assume $r_x:= d(x,\NN_X(x,\alpha z))$ for some $1<\alpha< c$ to be optimized later.

Recall that we fix an optimal set of centers $C^*$, and call the points in $X_z(C^*)$ ``true inliers'' and the points in $X\setminus X_z(C^*)$ ``true outliers''. Define $Y=O\cap X_z(C^*)$ as the set of true inliers removed by the algorithm by mistake, and $Z'=(X\setminus O) \cap (X\setminus X_z(C^*))$ as the set of true outliers included by the algorithm by mistake. We have $|Y|=|Z'|$, and again we assume that $|Y|\geq 1$. Let $T = \min_{x\in O} \{r_x\}$, i.e., the threshold separating the $r_x$ of $X\setminus O$ and $O$. 

The following two lemmas are analogous to~\Cref{lem:opt_lowerbound1} and ~\Cref{lem:opt_lowerbound2}.

\begin{lem}
\label{lem:opt_lowerbound1_}
    Suppose each optimal cluster has size at least $cz$, and $\min_{y\in Y}d(y,C^*)=\lambda T$ for some $\lambda\geq 0$, then $f_z(X,C^*) \geq |Y|T^2\cdot \max\{\lambda^2, (c-\alpha)(1-\lambda)^2\}$ when $\lambda\in[0,1]$, and $f_z(X,C^*) \geq |Y|T^2\cdot \lambda^2$ when $\lambda>1$.
\end{lem}
\begin{proof}
    Since every $y\in Y$ satisfies $d(y,C^*)\geq \lambda T$, we have
    \[
    f_z(X,C^*) \geq f_z(Y,C^*) = \sum_{y\in Y}d(y,C^*)^2 \geq |Y|T^2\cdot \lambda^2 ~.
    \]
    When $\lambda\in[0,1]$, 
    fix a point $y\in Y$ such that $d(y,C^*) =\lambda T$. Since $y\in O$, we have $r_y \geq T$, which implies $|B_X^\circ(y,T)| \leq \alpha z$. By the assumption that each optimal cluster has size at least $cz$, at least $(c-\alpha)z$ points in $\cluster_z(y,C^*)$ are not in $B_X^\circ(y,T)$, and for each such point $x$, by the triangle inequality, we have
    \[
    \begin{aligned}
    d(x,\closest(y,C^*)) \geq &~ d(x,y) - d(y,\closest(y,C^*))\\
    \geq &~ T-\lambda T = T\cdot (1-\lambda) ~,
    \end{aligned}
    \]
    therefore,
    \[
    \begin{aligned}
    f_z(X,C^*) \geq &~ f_z(\cluster_z(y,C^*) \setminus B_X^\circ(y,T), C^*) \\
    \geq &~ (c-\alpha)zT^2\cdot (1-\lambda)^2 \geq |Y|T^2\cdot (c-\alpha)(1-\lambda)^2 ~.
    \end{aligned}
    \]
    The last inequality follows from the fact that $|Y| \leq |O|= z$. \qedhere
\end{proof}

\begin{lem}
\label{lem:opt_lowerbound2_}
    Suppose each optimal cluster has size at least $cz$, and $\max_{o\in Z'}d(o,C^*) = \delta T$ for some $\delta>1$, then $f_z(X,C^*) \geq |Y|T^2\cdot (\alpha-1)(\delta-1)^2$.
\end{lem}
\begin{proof}
    Let $o = \arg\max_{o'\in Z'}\{d(o',C^*)\}$, so $d(o,C^*)=\delta T$. Since $o\in X\setminus O$, we have $r_o \leq T$, which implies $|B_X(o,T)|\geq \alpha z$. Therefore, since there are at most $z$ outliers, there must be at least $(\alpha-1)z$ true inliers in $B_X(o,T)$. For each such inlier $x$, we must have $d(x,C^*) \geq (\delta-1)T$, otherwise there would be the following contradiction: 
    \[
    d(o,C^*) \leq d(o,x)+d(x,C^*) < T + (\delta-1)T = \delta T ~.
    \]
    Thus, we must have
    \[
    \begin{aligned}
    f_z(X,C^*) \geq &~ \sum_{x\in B_X(o,T)\cap X_z(C^*)} d(x,C^*)^2
    \geq (\alpha-1)zT^2 \cdot (\delta-1)^2 \geq |Y|T^2\cdot (\alpha-1)(\delta-1)^2 ~. 
    \end{aligned}
    \]
\end{proof}

The following lemma is analogous to~\Cref{lem:cost_upperbound1}, but in the proof we parameterize the thresholds between the three cases to obtain a tight analysis.

\begin{lem}
\label{lem:cost_upperbound1_}
    Suppose each optimal cluster has size at least $cz$, then
    \[
    f(X\setminus O) \leq f_z(X,C^*)\cdot \frac{(1+\sqrt{(c-1)/2})^2}{(c-1)/2}(x^*)^2 ~.
    \]
    where $x^*$ is the unique real solution greater than $1$ to the following equation:
    \[
    \left[\left(1 + \sqrt{\frac{c-1}{2}}\right)^2x^2 - \frac{c-1}{2}\right](x-1)^2 - 2x+1=0 ~.
    \]
\end{lem}
\begin{proof}
    We distinguish between the following three cases: let $\delta^*>1$,

    \paragraph{Case 1} If $\max_{o\in Z'}d(o,C^*) \leq \delta^*T$, then by~\Cref{lem:opt_lowerbound1_}, assuming $\min_{y\in Y}d(y,C^*) = \lambda T$ for some $\lambda\geq 0$, then we have
    \[
    \begin{aligned}
        f(X\setminus O,C^*)
        \leq &~ f_z(X,C^*) + f(Z',C^*) - f(Y,C^*)\\
        \leq &~ f_z(X,C^*) + |Z'|\cdot (\delta^*T)^2 - |Y|\cdot (\lambda T)^2\\
        \leq &~ \begin{cases}
        f_z(X,C^*)\cdot \left(1 + \frac{(\delta^*)^2-\lambda^2}{\max\{\lambda^2, (c-\alpha)(1-\lambda)^2\}}\right), \text{if }0\leq \lambda\leq 1\\
        f_z(X,C^*) \cdot (\delta^*)^2/\lambda^2, \text{if }\lambda >1
        \end{cases}~.
    \end{aligned}
    \]
    
    \paragraph{Case 2} If $\max_{o\in Z'}d(o,C^*) = \delta T$ for some $\delta > \delta^*$, and $d(y,C^*) \geq (\delta-1)T$ for all $y\in Y$, then by~\Cref{lem:opt_lowerbound1_}, we have
    \[
    \begin{aligned}
        f(X\setminus O,C^*)
        \leq &~ f_z(X,C^*) + f(Z',C^*) - f(Y,C^*)\\
        \leq &~ f_z(X,C^*) + |Z'|\cdot (\delta T)^2 - |Y|\cdot ((\delta-1)T)^2\\
        \leq &~ f_z(X,C^*)\cdot \left(1 + \frac{2\delta-1}{\max\{1, \alpha-1\}\cdot (\delta-1)^2}\right)\\
        \leq &~ f_z(X,C^*)\cdot \left(1 + \frac{2\delta^*-1}{\max\{1, \alpha-1\}\cdot (\delta^*-1)^2}\right) ~.
    \end{aligned}
    \]
    The last inequality is due to the monotonicity of $\frac{2\delta-1}{(\delta-1)^2}$.
    
    \paragraph{Case 3} If $\max_{o\in Z'}d(o,C^*) = \delta T$ for some $\delta > \delta^*$, and $\min_{y\in Y}d(y,C^*) =\lambda T$ for some $\lambda < \delta-1$, then by~\Cref{lem:opt_lowerbound1_} and~\Cref{lem:opt_lowerbound2_}, we have
    \[
    \begin{aligned}
        f(X\setminus O,C^*)
        \leq &~ f_z(X,C^*) + f(Z',C^*) - f(Y,C^*)\\
        \leq &~ f_z(X,C^*) + |Z'|\cdot (\delta T)^2 - |Y|\cdot (\lambda T)^2\\
        \leq &~ \begin{cases}
        f_z(X,C^*)\cdot \left(1 + \frac{\delta^2-\lambda^2}{\max\{(c-\alpha)(1-\lambda)^2, (\alpha-1)(\delta-1)^2\}}\right), \text{if }0\leq \lambda\leq 1\\
        f_z(X,C^*) \cdot \left(1 + \frac{\delta^2-\lambda^2}{\max\{\lambda^2, (\alpha-1)(\delta-1)^2\}}\right), \text{if }\lambda >1
        \end{cases}~.
    \end{aligned}
    \]
    Here in the case of $0\leq \lambda\leq 1$, we omit the $\lambda^2$ term in the denominator as indicated by~\Cref{lem:opt_lowerbound1_}. This simplifies the analysis and does not affect the optimality of the derived bound.

    It remains to optimize over the choice of $\alpha$ and $\delta^*$ such that the maximum over the upper bounds in the three cases is minimized. It can be verified that the optimal choice is $\alpha = \frac{c+1}{2}$, then the upper bound in Case 1 is maximized at $\lambda^* = \frac{\sqrt{(c-1)/2}}{1 + \sqrt{(c-1)/2}}$ and takes value $U_1(\delta^*)=\frac{(1+\sqrt{(c-1)/2})^2}{(c-1)/2}(\delta^*)^2$; the upper bound in Case 2 is $U_2(\delta^*)=1 + \frac{2\delta^*-1}{(c-1)/2\cdot (\delta^*-1)^2}$; the upper bound in Case 3 turns out to be dominated by Case 2. When $\delta^*>1$, $U_1(\delta^*)$ is increasing in $\delta^*$ while $U_2(\delta^*)$ is decreasing in $\delta^*$, so the optimal choice of $\delta^*$ is given by $U_1(\delta^*)=U_2(\delta^*)$, which simplifies to 
    \[
    \left[\left(1 + \sqrt{\frac{c-1}{2}}\right)^2(\delta^*)^2 - \frac{c-1}{2}\right](\delta^*-1)^2 - (2\delta^*-1)=0 ~.
    \]
    Then in all three cases we have $f(X\setminus O,C^*) \leq f_z(X,C^*)\cdot U_1(\delta^*)$.
\end{proof}

\subsection{Proof of~\Cref{thm:main_result2}}
\label{app:proof_main_result2}

To justify using the sum of all $K$-nearest neighbors for $K\in[z+1,cz]$ instead of another interval, we further parameterize the algorithm: suppose we use the sum of distances to the $[lz+1,rz]$-th neighbors as a score (such that $1\leq l\leq r\leq c$). Assume $(r-l)zT$ is the threshold that separates the scores of $X\setminus O$ and $O$. We have the following two lower bounds on the cost of the optimal solution, which are analogous to~\Cref{lem:opt_lowerbound1_} and~\Cref{lem:opt_lowerbound2_}. The only difference is that the bounds are now stronger since we use a sum over multiple neighbors as the outlier score.

\begin{lem}
\label{lem:opt_lowerbound3_}
    Suppose each optimal cluster has size at least $cz$, and $\min_{y\in Y}d(y,C^*) = \lambda T$ for some $\lambda\geq 0$, then we have $f_z(X,C^*) \geq |Y|T^2\cdot \max\{\lambda^2, (c-l)(1-\lambda)^2\}$ when $\lambda\in[0,1]$, and $f_z(X,C^*) \geq |Y|T^2\cdot \lambda^2$ when $\lambda>1$.
\end{lem}
\begin{proof}
    Since every $y\in Y$ satisfies $d(y,C^*) \geq \lambda T$, we have
    \[
    f_z(X,C^*) \geq f_z(Y,C^*) = \sum_{y\in Y}d(y,C^*)^2 \geq |Y|T^2\cdot \lambda^2 ~.
    \]

    When $\lambda\in[0,1]$, fix a point $y\in Y$ such that $d(y,C^*) = \lambda T$. Let $x_i = \NN_{\cluster_z(y,C^*)}(y,lz+i)$. Using a similar argument as in the proof of~\Cref{lem:opt_lowerbound1_}, each of the $[rz+1,cz]$-th nearest neighbors of $y$ is at distance at least $T\cdot |1-\lambda|$ from $C^*$, therefore,
    \[
    \sum_{i=(r-l)z+1}^{(c-l)z}d(x_i,C^*)^2 \geq (c-r)zT^2\cdot (1-\lambda)^2 ~.
    \]
    
    Further, by the assumption that $y\in O$, we have $\sum_{i=1}^{(r-l)z}d(x_i,y) \geq (r-l)zT$. Then, by the triangle inequality and Jensen's inequality, we have
    \[
    \begin{aligned}
    \sum_{i=1}^{(r-l)z} d(x_i,C^*)^2 \geq &~ \sum_{i=1}^{(r-l)z} (d(x_i,y)-d(y,C^*))^2 \\
    \geq &~ (r-l)z\cdot \left(\frac{1}{(r-l)z}\sum_{i=1}^{(r-l)z} (d(x_i,y)-d(y,C^*))\right)^2\\
    \geq &~ (r-l)z\cdot (T-d(y,C^*))^2\\
    \geq &~ (r-l)zT^2\cdot \left(1-\lambda\right)^2\\
    \end{aligned}
    \]

    Therefore, 
    \[
    \begin{aligned}
    f_z(X,C^*) \geq &~ \sum_{i=1}^{(c-l)z} d(x_i,C^*)^2 = \sum_{i=1}^{(r-l)z} d(x_i,C^*)^2 + \sum_{i=(r-l)z+1}^{(c-l)z}d(x_i,C^*)^2\\
    \geq &~ (c-l)zT^2 \cdot (1-\lambda)^2 \geq |Y|T^2\cdot (c-l)(1-\lambda)^2 ~.
    \end{aligned}
    \]
    The last inequality follows from the fact that $|Y| \leq |O|= z$. \qedhere
\end{proof}

\begin{lem}
\label{lem:opt_lowerbound4_}
    Suppose each optimal cluster has size at least $cz$, and $\max_{o\in Z'}d(o,C^*) = \delta T$ for some $\delta>1$, then $f_z(X,C^*) \geq |Y|T^2\cdot (r-l)(\delta-1)^2$.
\end{lem}
\begin{proof}
    Let $o = \arg\max_{o'\in Z'}\{d(o',C^*)\}$, so $d(o,C^*)=\delta T$. Let $x_i$ denote the $i$-th closest true inlier to $o$, then since there are at most $z$ true outliers, we have
    \[
    \sum_{i=(l-1)z+1}^{(r-1)z} d(x_i,o) \leq \sum_{i=lz+1}^{rz}d(o, \NN_X(o,i)) \leq (r-l)zT ~.
    \]
    
    Let $C_i = \closest(x_i,C^*)$, then by the triangle inequality, we have
    \[
    d(C_i,x_i)+d(x_i,o) \geq d(C_i,o),\ \forall i\in[cz] ~,
    \]
    therefore, by Jensen's inequality, 
    \[
    \begin{aligned}
    \sum_{i=(l-1)z+1}^{(r-1)z} d(C_i,x_i)^2 
    \geq &~ \sum_{i=(l-1)z+1}^{(r-1)z} (\max\{d(C_i,o)-d(x_i,o), 0\})^2\\
    \geq &~ (r-l)z\cdot \left(\max\left\{\frac{1}{(r-l)z}\sum_{i=(l-1)z+1}^{(r-1)z}(d(C_i,o)-d(x_i,o)), 0\right\}\right)^2\\
    \geq &~ (r-l)z\cdot \left(\max\left\{d(o,C^*)-d(x_i,o), 0\right\}\right)^2\\
    \geq &~ (r-l)z\cdot \left(\delta T-T\right)^2=  (r-l)zT^2\cdot (\delta-1)^2 ~.
    \end{aligned} 
    \]
    
    Therefore, 
    \[
    \begin{aligned}
    f_z(X,C^*) \geq &~ \sum_{i=(l-1)z+1}^{(r-1)z} d(C_i,x_i)^2 \geq (r-l)zT^2\cdot (\delta-1)^2 \geq |Y|T^2\cdot (r-l)(\delta-1)^2 ~.
    \end{aligned}
    \]
\end{proof}

\begin{lem}
\label{lem:cost_upperbound2_}
    Suppose each optimal cluster has size at least $cz$, then
    \[
    f(X\setminus O) \leq f_z(X,C^*)\cdot \frac{(1+\sqrt{c-1})^2}{c-1}(x^*)^2
    \]
    where $x^*$ is the unique real root greater than $1$ of the following equation:
    \[
    \left[\left(1 + \sqrt{c-1}\right)^2x^2 - (c-1)\right](x-1)^2 - 2x+1=0 ~.
    \]
\end{lem}
\begin{proof}
    We again distinguish between the following three cases: let $\delta^*>1$,

    \paragraph{Case 1} If $\max_{o\in Z'}d(o,C^*) \leq \delta^*T$, then by~\Cref{lem:opt_lowerbound3_}, for all $\lambda>0$, assume $\min_{y\in Y}d(y,C^*) = \lambda T$ then we have
    \[
    \begin{aligned}
        f(X\setminus O,C^*)
        \leq &~ f_z(X,C^*) + f(Z',C^*) - f(Y,C^*)\\
        \leq &~ f_z(X,C^*) + |Z'|\cdot (\delta^*T)^2 - |Y|\cdot (\lambda T)^2\\
        \leq &~ \begin{cases}
        f_z(X,C^*)\cdot \left(1 + \frac{(\delta^*)^2-\lambda^2}{\max\{\lambda^2, (c-l)(1-\lambda)^2\}}\right), \text{if }0\leq \lambda\leq 1\\
        f_z(X,C^*) \cdot (\delta^*)^2/\lambda^2, \text{if }\lambda >1
        \end{cases}~.
    \end{aligned}
    \]
    
    \paragraph{Case 2} If $\max_{o\in Z'}d(o,C^*) = \delta T$ for some $\delta > \delta^*$, and $d(y,C^*) \geq (\delta-1)T$ for all $y\in Y$, then by~\Cref{lem:opt_lowerbound3_}, we have
    \[
    \begin{aligned}
        f(X\setminus O,C^*)
        \leq &~ f_z(X,C^*) + f(Z',C^*) - f(Y,C^*)\\
        \leq &~ f_z(X,C^*) + |Z'|\cdot (\delta T)^2 - |Y|\cdot ((\delta-1)T)^2\\
        \leq &~ f_z(X,C^*)\cdot \left(1 + \frac{2\delta-1}{\max\{1, r-l\}\cdot (\delta-1)^2}\right)\\
        \leq &~ f_z(X,C^*)\cdot \left(1 + \frac{2\delta^*-1}{\max\{1, r-l\}\cdot (\delta^*-1)^2}\right) ~.
    \end{aligned}
    \]
    
    \paragraph{Case 3} If $\max_{o\in Z'}d(o,C^*) = \delta T$ for some $\delta > \delta^*$, and $\min_{y\in Y}d(y,C^*) =\lambda T$ for some $\lambda < \delta-1$, then by~\Cref{lem:opt_lowerbound3_} and~\Cref{lem:opt_lowerbound4_}, we have
    \[
    \begin{aligned}
        f(X\setminus O,C^*)
        \leq &~ f_z(X,C^*) + f(Z',C^*) - f(Y,C^*)\\
        \leq &~ f_z(X,C^*) + |Z'|\cdot (\delta T)^2 - |Y|\cdot (\lambda T)^2\\
        \leq &~ \begin{cases}
        f_z(X,C^*)\cdot \left(1 + \frac{\delta^2-\lambda^2}{\max\{(c-l)(1-\lambda)^2, (r-l)(\delta-1)^2\}}\right), \text{if }0\leq \lambda\leq 1\\
        f_z(X,C^*) \cdot \left(1 + \frac{\delta^2-\lambda^2}{\max\{\lambda^2, (r-l)(\delta-1)^2\}}\right), \text{if }\lambda >1
        \end{cases}~.
    \end{aligned}
    \]

    It is straightforward to see that to minimize the upper bounds, we should choose $r=c$ and $l=1$. It remains to optimize the choice of $\delta^*$ such that the maximum over the upper bounds in the three cases is minimized. Similar to~\Cref{lem:cost_upperbound1_}, the upper bound in Case 1 is maximized at $\lambda^* = \frac{\sqrt{c-1}}{1 + \sqrt{c-1}}$ and takes value $V_1(\delta^*)=\frac{(1+\sqrt{c-1})^2}{c-1}(\delta^*)^2$; the upper bound in Case 2 is $V_2(\delta^*)=1 + \frac{2\delta^*-1}{(c-1)\cdot (\delta^*-1)^2}$, and the upper bound in Case 3 is dominated by Case 2. The optimal choice of $\delta^*$ is given by $V_1(\delta^*)=V_2(\delta^*)$, which simplifies to 
    \[
    \left[\left(1 + \sqrt{c-1}\right)^2(\delta^*)^2 - (c-1)\right](\delta^*-1)^2 - (2\delta^*-1)=0 ~.
    \]
    Then in all three cases we have $f(X\setminus O,C^*) \leq f_z(X,C^*)\cdot V_1(\delta^*)$.
\end{proof}

Finally, we present a comparison of the approximation ratios in~\Cref{thm:main_result} and~\Cref{thm:main_result2} 

\begin{table}[h]
\caption{Approximation ratio for different $c$ for the Robust $k$-Means objective.}\label{tab:c_ratio}
\centering
\begin{tabular}{ccccccc}
\toprule
$c$-value & 2 & 3 & 4 & 5 & 10 \\
\midrule
\Cref{thm:main_result} & 14.30 & 9 & 7.04 & 5.98 & 3.99 \\
\Cref{thm:main_result2} & 9 & 5.98 & 4.84 & 4.21 & 2.96 \\
\bottomrule
\end{tabular}
\end{table}

\section{Extensions to $k$-Median and $k$-Center}
\label{app:extension}

This section analyzes the performance of our proposed methods when applied to the Robust $k$-Median and Robust $k$-Center objective. The extension to Robust $k$-Median is a straightforward adaptation of the proof strategy for Robust $k$-Means. 
The purpose of this section is to demonstrate that the proposed framework is not specific to Robust $k$-Means, and can simultaneously yield non-trivial theoretical guarantees across multiple robust clustering objectives.

\subsection{Robust $k$-Median}

Our proofs for~\Cref{thm:main_result} and~\Cref{thm:main_result2} can be directly extended to the Robust $k$-Median objective, by replacing all the squared distances by unsquared distances in the analysis. Here we present a formal upper bound of applying~\Cref{alg:OKMeans2} to the Robust $k$-Median objective, whose proof is similar to~\Cref{thm:main_result2} and is omitted.

\begin{thm}
\label{thm:extension_kMedian}
    Let $I=(X,k,z)$ be an instance of robust $k$-Median and $C^*$ be the optimal set of centers. Suppose $\mathcal{A}$ is a $\beta$-approximation algorithm for $k$-Median, and each cluster induced by $C^*$ has size at least $cz$, then there is an algorithm that calls $\mathcal{A}$ once and computes a set of centers $C$ such that 
    \[
    g_{z}(X,C) \leq \zeta(c)\beta\cdot g_z(X,C^*) ~,
    \]
    where $\zeta(c)=\max\left\{\frac{2c-1+\sqrt{4c+1}}{2(c-1)},\frac{c+1}{c-1}\right\}$.
\end{thm}

In the following table we present the value of $\zeta(c)$ for some values of $c$.

\begin{table}[h]
\caption{Approximation ratio for different $c$ for the Robust $k$-Median objective.}\label{tab:c_ratio_kMedian}
\centering
\begin{tabular}{ccccccc}
\toprule
$c$-value & 2 & 3 & 4 & 5 & 10 \\
\midrule
\Cref{thm:extension_kMedian} & 3 & 2.15 & 1.85 & 1.70 & 1.41 \\
\bottomrule
\end{tabular}
\end{table}

\subsection{Robust $k$-Center}

For the $k$-Center objective, both our methods achieve a $3$-approximate reduction. We give a brief justification for~\Cref{alg:OKMeans}. Let $r_{\max}$ be the largest radius among the optimal clusters, we have the following two properties.
\begin{observation}
    We have $r_x \leq 2r_{\max}$ for each true inlier $x$.
\end{observation} 
\begin{proof}
    $B(x,2r_{\max})$ already completely contains $\cluster_z(x,C^*)$, which consists of at least $3z$ points.
\end{proof}
\begin{observation}
    If $r_o \leq 2r_{\max}$ for some true outlier $o$, then $d(r_o, C^*) \leq 3r_{\max}$.
\end{observation}
\begin{proof}
    Since there are only $z$ outliers, $B(o,r_o)$ must contain at least one true inlier $x$, so if $r_o\leq 2r_{\max}$ then we must have
    \[
    d(o,\closest(x,C^*)) \leq d(o,x) + d(x,\closest(x,C^*)) \leq 2r_{\max}+r_{\max} = 3r_{\max} ~.  \qedhere
    \]
\end{proof}

The analysis for~\Cref{alg:OKMeans2} is similar.

However, since \citet{CharikarKMN01} already established an optimal $3$-approximate algorithm for robust $k$-Center without assumption on the optimal cluster sizes (they also showed that better-than-$3$ approximation is NP-hard), we chose not to elaborate on this aspect of our algorithm in the current paper. 

\newpage
\section{Additional Experimental Results}
\label{app:figures}

This section provides more experimental results and plots.

\begin{table*}[!ht]
\centering
\caption{Comparison results of robust $k$-means algorithms on each dataset. }
\label{tab:baseline_mean_std}
\footnotesize
\setlength{\tabcolsep}{3.2pt}
\renewcommand{\arraystretch}{0.92}
\resizebox{\textwidth}{!}{%
\begin{tabular}{l c c c @{\hspace{5mm}} l c c c}
\toprule
Method & Cost Mean $\pm$ Std $\downarrow$ & Recall $\uparrow$ & Time(s) $\downarrow$ & Method & Cost Mean $\pm$ Std $\downarrow$ & Recall $\uparrow$ & Time(s) $\downarrow$ \\
\midrule
\multicolumn{4}{c}{\textbf{SKIN-5}} & \multicolumn{4}{c}{\textbf{SKIN-10}} \\
\cmidrule(lr){1-4}\cmidrule(lr){5-8}
TIKMeans & \best{\scicost{6.151}{0.194}{4}} & \pmstd{0.742}{0.016} & \pmstd{3.314}{0.495} & TIKMeans & \best{\scicost{6.608}{0.131}{4}} & \pmstd{0.940}{0.002} & \pmstd{2.965}{0.178} \\
IKMeans & \scicost{7.109}{0.885}{4} & \best{\pmstd{0.767}{0.020}} & \second{\pmstd{1.444}{0.161}} & IKMeans & \scicost{7.512}{1.133}{4} & \second{\pmstd{0.950}{0.013}} & \second{\pmstd{1.368}{0.060}} \\
RobustKmeans++ & \scicost{7.620}{0.711}{4} & \pmstd{0.741}{0.033} & \pmstd{1.618}{0.012} & RobustKmeans++ & \scicost{8.857}{1.237}{4} & \pmstd{0.936}{0.010} & \pmstd{1.617}{0.034} \\
NKMeans & \scicost{6.887}{0.646}{4} & \pmstd{0.744}{0.015} & \pmstd{4.402}{0.105} & NKMeans & \scicost{7.057}{0.275}{4} & \pmstd{0.944}{0.016} & \pmstd{4.375}{0.069} \\
KMeans++ & \scicost{7.522}{1.464}{4} & \pmstd{0.749}{0.029} & \best{\pmstd{0.476}{0.037}} & KMeans++ & \scicost{1.018}{0.149}{5} & \pmstd{0.891}{0.013} & \best{\pmstd{0.480}{0.039}} \\
OKMeans (Ours) & \scicost{6.494}{0.312}{4} & \second{\pmstd{0.757}{0.011}} & \pmstd{2.056}{0.052} & OKMeans (Ours) & \scicost{6.982}{0.671}{4} & \pmstd{0.949}{0.013} & \pmstd{2.045}{0.045} \\
OKMeans2 (Ours) & \second{\scicost{6.410}{0.300}{4}} & \pmstd{0.746}{0.014} & \pmstd{2.395}{0.022} & OKMeans2 (Ours) & \second{\scicost{6.707}{0.166}{4}} & \best{\pmstd{0.955}{0.011}} & \pmstd{2.441}{0.057} \\
\midrule
\multicolumn{4}{c}{\textbf{SUSY-5}} & \multicolumn{4}{c}{\textbf{SUSY-10}} \\
\cmidrule(lr){1-4}\cmidrule(lr){5-8}
TIKMeans & \scicost{5.165}{0.048}{7} & \pmstd{0.759}{0.035} & \pmstd{94.824}{1.276} & TIKMeans & \scicost{5.272}{0.040}{7} & \pmstd{0.975}{0.002} & \pmstd{89.186}{4.053} \\
IKMeans & \scicost{5.279}{0.104}{7} & \pmstd{0.782}{0.049} & \pmstd{46.489}{2.589} & IKMeans & \scicost{5.334}{0.087}{7} & \pmstd{0.976}{0.004} & \pmstd{42.134}{2.536} \\
RobustKmeans++ & \scicost{5.791}{0.132}{7} & \pmstd{0.708}{0.072} & \pmstd{74.735}{2.058} & RobustKmeans++ & \scicost{5.966}{0.130}{7} & \pmstd{0.973}{0.006} & \pmstd{74.092}{1.903} \\
NKMeans & \second{\scicost{4.825}{0.022}{7}} & \pmstd{0.739}{0.014} & \pmstd{17.346}{0.281} & NKMeans & \scicost{4.941}{0.028}{7} & \pmstd{0.972}{0.002} & \pmstd{17.529}{0.456} \\
KMeans++ & \scicost{4.885}{0.087}{7} & \best{\pmstd{0.831}{0.012}} & \best{\pmstd{9.725}{0.111}} & KMeans++ & \scicost{5.173}{0.235}{7} & \pmstd{0.980}{0.004} & \best{\pmstd{9.728}{0.152}} \\
OKMeans (Ours) & \scicost{4.830}{0.064}{7} & \second{\pmstd{0.799}{0.020}} & \second{\pmstd{12.840}{0.353}} & OKMeans (Ours) & \best{\scicost{4.873}{0.023}{7}} & \second{\pmstd{0.982}{0.002}} & \second{\pmstd{12.919}{0.432}} \\
OKMeans2 (Ours) & \best{\scicost{4.801}{0.034}{7}} & \pmstd{0.798}{0.018} & \pmstd{13.469}{0.450} & OKMeans2 (Ours) & \second{\scicost{4.876}{0.020}{7}} & \best{\pmstd{0.982}{0.002}} & \pmstd{13.407}{0.286} \\
\midrule
\multicolumn{4}{c}{\textbf{SHUTTLE}} & \multicolumn{4}{c}{\textbf{KDDFULL}} \\
\cmidrule(lr){1-4}\cmidrule(lr){5-8}
TIKMeans & \best{\scicost{6.122}{0.260}{4}} & \pmstd{0.141}{0.142} & \pmstd{0.543}{0.008} & TIKMeans & \best{\scicost{3.320}{0.002}{7}} & \pmstd{0.601}{0.003} & \pmstd{49.850}{1.626} \\
IKMeans & \scicost{7.277}{0.503}{4} & \pmstd{0.141}{0.142} & \best{\pmstd{0.243}{0.003}} & IKMeans & \scicost{3.937}{0.861}{7} & \pmstd{0.615}{0.022} & \pmstd{21.419}{0.834} \\
RobustKmeans++ & \scicost{7.830}{1.193}{4} & \pmstd{0.094}{0.121} & \pmstd{0.363}{0.011} & RobustKmeans++ & \scicost{4.845}{1.260}{7} & \pmstd{0.611}{0.040} & \pmstd{40.355}{1.482} \\
NKMeans & \scicost{7.260}{0.353}{4} & \best{\pmstd{0.171}{0.061}} & \pmstd{4.785}{0.095} & NKMeans & \second{\scicost{3.466}{0.178}{7}} & \best{\pmstd{0.624}{0.013}} & \pmstd{14.859}{0.127} \\
KMeans++ & \scicost{7.055}{0.705}{4} & \pmstd{0.082}{0.127} & \second{\pmstd{0.294}{0.082}} & KMeans++ & \scicost{6.123}{2.152}{7} & \pmstd{0.553}{0.113} & \best{\pmstd{6.894}{0.104}} \\
OKMeans (Ours) & \second{\scicost{6.823}{0.625}{4}} & \second{\pmstd{0.171}{0.116}} & \pmstd{0.785}{0.074} & OKMeans (Ours) & \scicost{4.019}{0.770}{7} & \second{\pmstd{0.616}{0.010}} & \second{\pmstd{9.510}{0.062}} \\
OKMeans2 (Ours) & \scicost{6.823}{0.625}{4} & \pmstd{0.171}{0.116} & \pmstd{0.776}{0.058} & OKMeans2 (Ours) & \scicost{4.250}{0.919}{7} & \pmstd{0.613}{0.013} & \pmstd{10.269}{0.574} \\
\bottomrule
\end{tabular}%
}
\end{table*}

\begin{table*}[!ht]
\centering
\caption{Performance of the KNN heuristic with constant $K$. }
\label{tab:param_okmeans_bad}
\footnotesize
\setlength{\tabcolsep}{3.5pt}
\renewcommand{\arraystretch}{1.2}
\begin{tabular}{ll cccc}
\toprule
Dataset & Metric & KNN ($K=2$) & KNN ($K=5$) & KNN ($K=10$) & OKMeans \\
\midrule
\multirow{3}{*}{\textbf{SKIN-5}} & Cost & \scicost{6.290}{0.275}{4} & \scicost{6.822}{0.417}{4} & \scicost{6.577}{0.214}{4} & \scicost{6.494}{0.312}{4} \\
 & Recall & \pmstd{0.760}{0.008} & \pmstd{0.758}{0.016} & \pmstd{0.758}{0.012} & \pmstd{0.757}{0.011} \\
 & Time (s) & \pmstd{1.907}{0.178} & \pmstd{1.859}{0.130} & \pmstd{1.929}{0.139} & \pmstd{2.056}{0.052} \\
\midrule
\multirow{3}{*}{\textbf{SKIN-10}} & Cost & \scicost{6.838}{0.385}{4} & \scicost{6.911}{0.585}{4} & \scicost{6.773}{0.449}{4} & \scicost{6.982}{0.671}{4} \\
 & Recall & \pmstd{0.953}{0.014} & \pmstd{0.950}{0.016} & \pmstd{0.953}{0.012} & \pmstd{0.949}{0.013} \\
 & Time (s) & \pmstd{1.888}{0.151} & \pmstd{1.864}{0.149} & \pmstd{1.927}{0.175} & \pmstd{2.045}{0.045} \\
\midrule
\multirow{3}{*}{\textbf{SUSY-5}} & Cost & \scicost{4.815}{0.024}{7} & \scicost{4.804}{0.022}{7} & \scicost{4.790}{0.030}{7} & \scicost{4.830}{0.064}{7} \\
 & Recall & \pmstd{0.832}{0.009} & \pmstd{0.826}{0.011} & \pmstd{0.824}{0.010} & \pmstd{0.799}{0.020} \\
 & Time (s) & \pmstd{11.935}{0.203} & \pmstd{11.867}{0.272} & \pmstd{11.892}{0.213} & \pmstd{12.840}{0.353} \\
\midrule
\multirow{3}{*}{\textbf{SUSY-10}} & Cost & \scicost{4.890}{0.039}{7} & \scicost{4.891}{0.065}{7} & \scicost{4.897}{0.062}{7} & \scicost{4.873}{0.023}{7} \\
 & Recall & \pmstd{0.982}{0.002} & \pmstd{0.982}{0.002} & \pmstd{0.982}{0.002} & \pmstd{0.982}{0.002} \\
 & Time (s) & \pmstd{11.805}{0.242} & \pmstd{11.968}{0.231} & \pmstd{11.841}{0.262} & \pmstd{12.919}{0.432} \\
\midrule
\multirow{3}{*}{\textbf{SHUTTLE}} & Cost & \scicost{6.717}{0.679}{4} & \scicost{7.152}{0.459}{4} & \scicost{6.711}{0.707}{4} & \scicost{6.823}{0.625}{4} \\
 & Recall & \pmstd{0.124}{0.116} & \pmstd{0.171}{0.113} & \pmstd{0.147}{0.124} & \pmstd{0.171}{0.116} \\
 & Time (s) & \pmstd{1.769}{0.153} & \pmstd{1.752}{0.172} & \pmstd{1.741}{0.139} & \pmstd{0.785}{0.074} \\
\midrule
\multirow{3}{*}{\textbf{KDDFULL}} & Cost & \scicost{4.612}{1.244}{7} & \scicost{4.371}{1.087}{7} & \scicost{4.631}{1.366}{7} & \scicost{4.019}{0.770}{7} \\
 & Recall & \pmstd{0.560}{0.089} & \pmstd{0.521}{0.120} & \pmstd{0.569}{0.074} & \pmstd{0.616}{0.010} \\
 & Time (s) & \pmstd{9.141}{0.568} & \pmstd{9.273}{0.458} & \pmstd{9.104}{0.605} & \pmstd{9.510}{0.062} \\
\bottomrule
\end{tabular}
\end{table*}

\begin{table*}[!ht]
\centering
\caption{Parameter sensitivity analysis of OKMeans and OKMeans2 under different values of $c$.}
\label{tab:param_scaling}
\footnotesize
\setlength{\tabcolsep}{2.8pt} 
\renewcommand{\arraystretch}{1.1}
\resizebox{\textwidth}{!}{
\begin{tabular}{l l c c c c c c c}
\toprule
Dataset & Method & $c=3$ & $c=4$ & $c=5$ & $c=7$ & $c=10$ & $c=15$ & $c=20$ \\
\midrule
\multirow{2}{*}{\textbf{SKIN-5}} 
 & OKMeans & $6.688 \times 10^{4}$ & $6.479 \times 10^{4}$ & \second{6.347 \times 10^{4}} & $6.605 \times 10^{4}$ & $6.613 \times 10^{4}$ & \best{6.257 \times 10^{4}} & $6.736 \times 10^{4}$ \\
 & OKMeans2 & $6.670 \times 10^{4}$ & \best{6.367 \times 10^{4}} & $6.624 \times 10^{4}$ & \second{6.533 \times 10^{4}} & $6.544 \times 10^{4}$ & $6.691 \times 10^{4}$ & $6.708 \times 10^{4}$ \\
\midrule
\multirow{2}{*}{\textbf{SKIN-10}} 
 & OKMeans & $7.053 \times 10^{4}$ & $7.013 \times 10^{4}$ & \second{6.821 \times 10^{4}} & \best{6.778 \times 10^{4}} & $6.963 \times 10^{4}$ & $7.144 \times 10^{4}$ & $6.836 \times 10^{4}$ \\
 & OKMeans2 & $6.925 \times 10^{4}$ & \best{6.682 \times 10^{4}} & \second{6.824 \times 10^{4}} & $6.980 \times 10^{4}$ & $7.006 \times 10^{4}$ & $7.114 \times 10^{4}$ & $6.996 \times 10^{4}$ \\
\midrule
\multirow{2}{*}{\textbf{SUSY-5}} 
 & OKMeans & \second{4.830 \times 10^{7}} & $4.861 \times 10^{7}$ & \best{4.814 \times 10^{7}} & $4.843 \times 10^{7}$ & $4.845 \times 10^{7}$ & $4.847 \times 10^{7}$ & $4.864 \times 10^{7}$ \\
 & OKMeans2 & \best{4.801 \times 10^{7}} & $4.863 \times 10^{7}$ & \second{4.823 \times 10^{7}} & $4.849 \times 10^{7}$ & $4.823 \times 10^{7}$ & $4.841 \times 10^{7}$ & $4.830 \times 10^{7}$ \\
\midrule
\multirow{2}{*}{\textbf{SUSY-10}} 
 & OKMeans & $4.873 \times 10^{7}$ & $4.868 \times 10^{7}$ & $4.868 \times 10^{7}$ & \best{4.865 \times 10^{7}} & \second{4.868 \times 10^{7}} & $4.891 \times 10^{7}$ & $4.882 \times 10^{7}$ \\
 & OKMeans2 & $4.876 \times 10^{7}$ & $4.873 \times 10^{7}$ & \second{4.866 \times 10^{7}} & $4.867 \times 10^{7}$ & \best{4.866 \times 10^{7}} & $4.880 \times 10^{7}$ & $4.888 \times 10^{7}$ \\
\midrule
\multirow{2}{*}{\textbf{SHUTTLE}} 
 & OKMeans & \best{6.823 \times 10^{4}} & \second{6.823 \times 10^{4}} & $6.823 \times 10^{4}$ & $6.823 \times 10^{4}$ & $6.889 \times 10^{4}$ & $6.998 \times 10^{4}$ & $6.998 \times 10^{4}$ \\
 & OKMeans2 & \best{6.823 \times 10^{4}} & \second{6.823 \times 10^{4}} & $6.823 \times 10^{4}$ & $6.823 \times 10^{4}$ & $6.823 \times 10^{4}$ & $6.823 \times 10^{4}$ & $6.998 \times 10^{4}$ \\
\midrule
\multirow{2}{*}{\textbf{KDDFULL}} 
 & OKMeans & $4.019 \times 10^{7}$ & $4.053 \times 10^{7}$ & $4.395 \times 10^{7}$ & \best{3.586 \times 10^{7}} & $4.143 \times 10^{7}$ & $4.251 \times 10^{7}$ & \second{3.640 \times 10^{7}} \\
 & OKMeans2 & $4.250 \times 10^{7}$ & $4.082 \times 10^{7}$ & $3.988 \times 10^{7}$ & $4.302 \times 10^{7}$ & \second{3.967 \times 10^{7}} & \best{3.770 \times 10^{7}} & $4.087 \times 10^{7}$ \\
\bottomrule
\end{tabular}
}
\end{table*}




\begin{figure}[!ht]
    \centering
    \includegraphics[width=\textwidth]{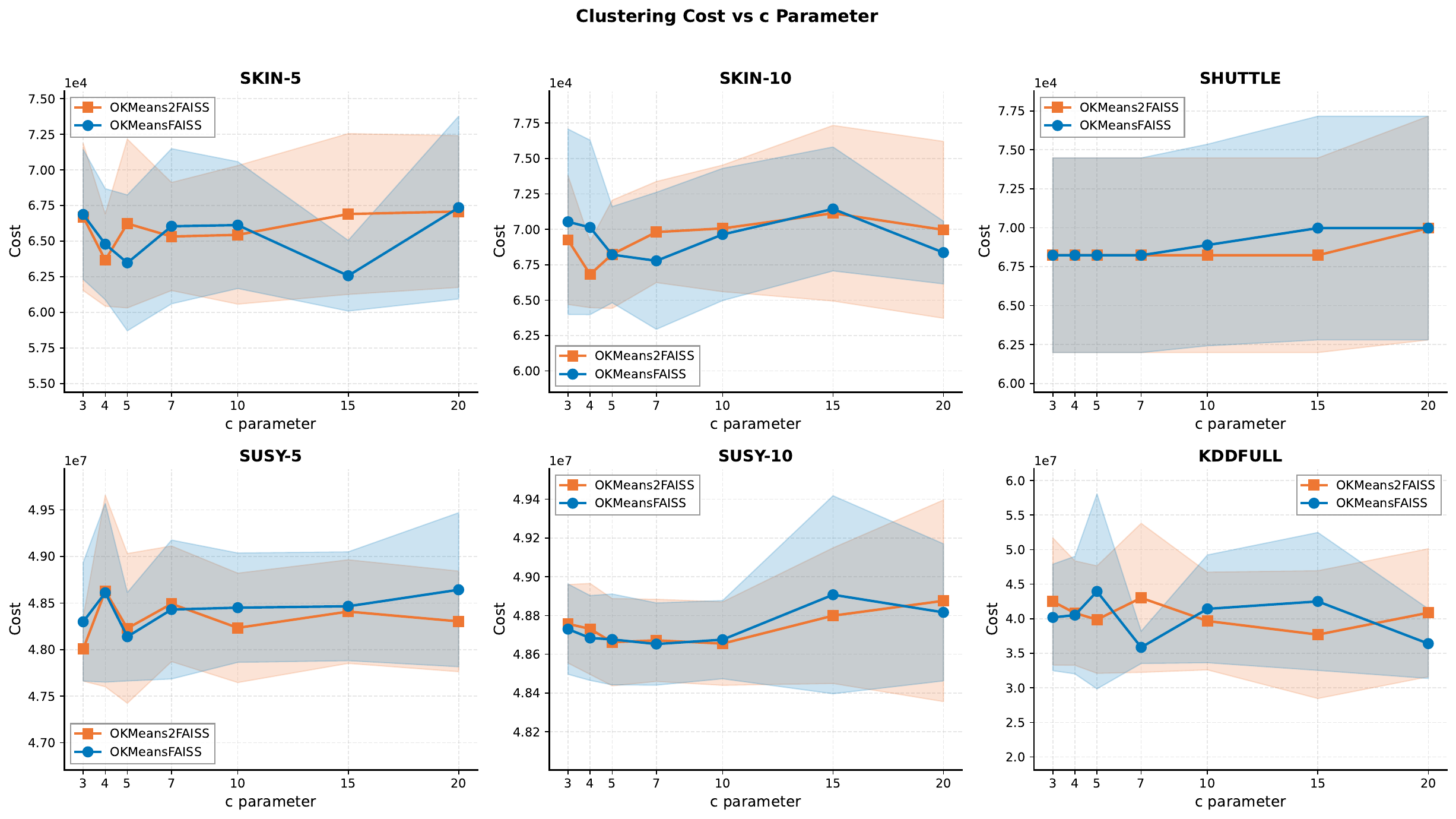}
    \caption{Parameter Sensitivity Cost Comparison}
    \label{fig:param_cost_comparison}
\end{figure}

\begin{figure}[!ht]
    \centering
    \includegraphics[width=\textwidth]{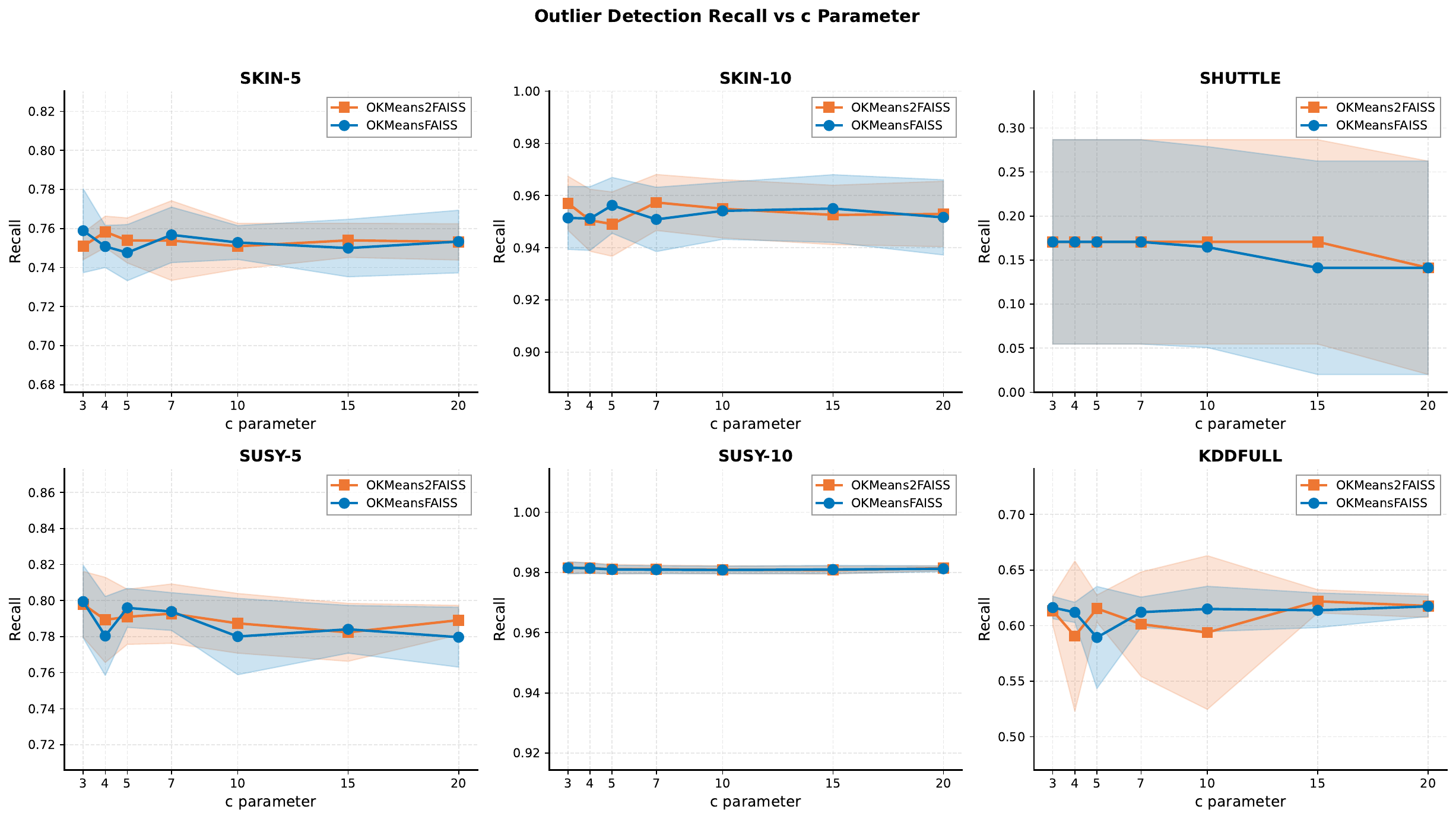}
    \caption{Parameter Sensitivity Recall Comparison}
    \label{fig:param_recall_comparison}
\end{figure}

\begin{figure}[!ht]
    \centering
    \includegraphics[width=\textwidth]{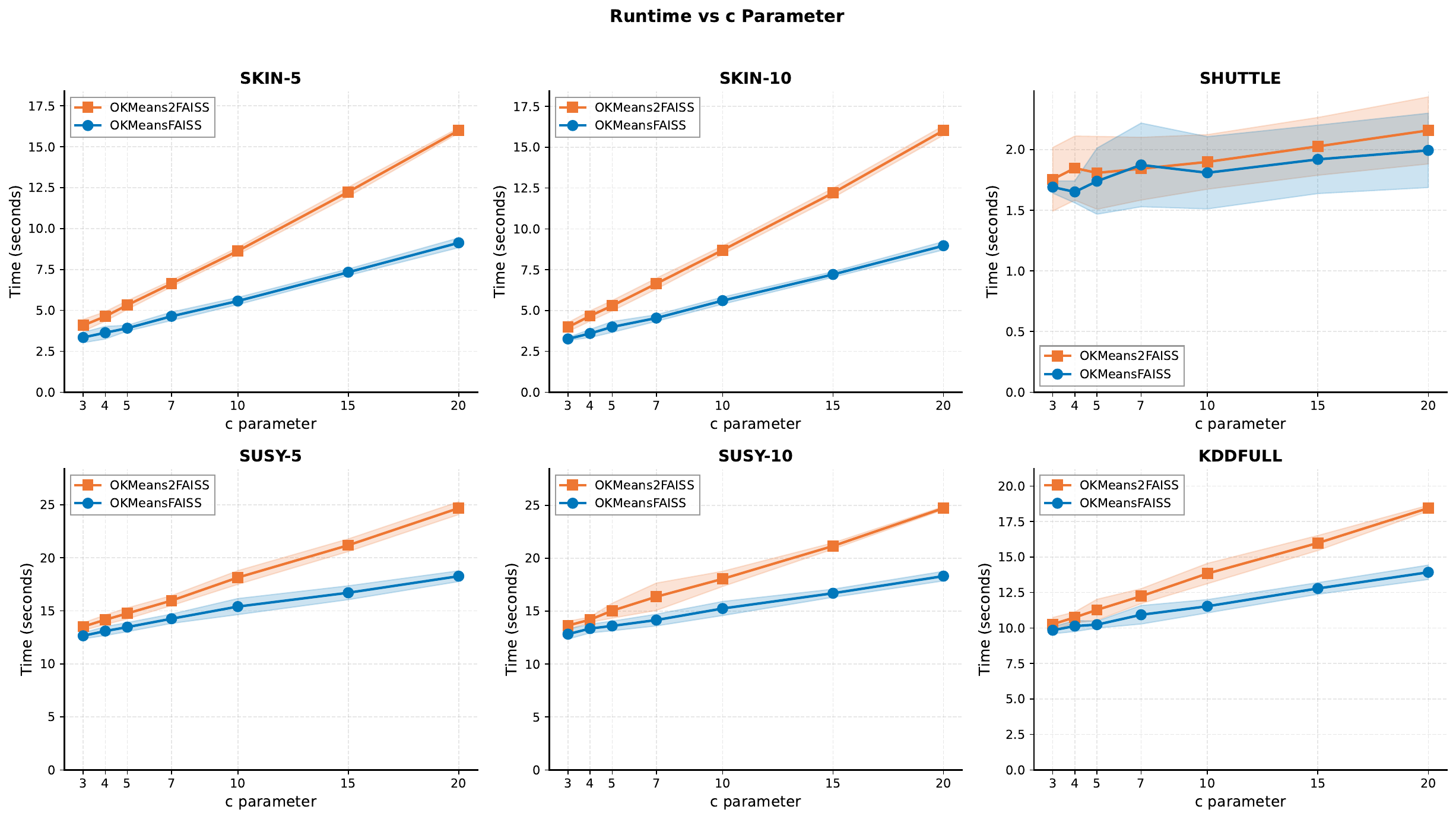}
    \caption{Parameter Sensitivity Time Comparison}
    \label{fig:param_time_comparison}
\end{figure}

\clearpage

\end{document}